\documentclass{article}

\PassOptionsToPackage{round}{natbib}
\usepackage[final]{neurips_2022}
\usepackage{float}

\usepackage[utf8]{inputenc} 
\usepackage[T1]{fontenc}    
\usepackage{hyperref}       
\usepackage{url}            
\usepackage{booktabs}       
\usepackage{amsfonts}       
\usepackage{nicefrac}       
\usepackage{amsmath,amsthm,amssymb}
\usepackage{mathtools}
\usepackage{microtype}      
\usepackage{xcolor}         
\usepackage{subcaption}
\usepackage{mathrsfs}
\usepackage{hyperref}
\usepackage{pifont}
\usepackage{soul}
\usepackage{comment}
\usepackage[ruled,vlined,noend]{algorithm2e}

\definecolor{red}{RGB}{255,0,0}
\definecolor{green}{RGB}{0,150,0}

\newtheorem{thmlem}{Lemma}
\newtheorem{thmcol}{Corollary}

\newtheorem{thmprop}{Proposition}
\newtheorem{thmthm}{Theorem}

\newtheorem{thmasmp}{Assumption}

\theoremstyle{definition}
\newtheorem{thmdef}{Definition}
\newtheorem{thmex}{Example}
\newtheorem{thmrem}{Remark}
\newtheorem*{thmrems*}{Remarks}

\newenvironment{proofsketch}{%
  \proof}{\endproof}

\DeclareMathOperator*{\argmin}{arg\,min}

\definecolor{verylightgray}{rgb}{0.95, 0.95, 0.95}

\definecolor{verylightgreen}{rgb}{0.94, 0.97, 0.92}

\def\E{\mathbb{E}}
\def\V{\mathrm{Var}}

\def\scrA{\mathscr{A}}
\def\PI{\textsc{p}}
\def\CL{\textsc{c}}
\def\SRL{\textsc{srl}}
\def\CRL{\textsc{crl}}
\def\GRL{\textsc{grl}}

\def\traffic{\textbf{Traffic}}
\def\squaresign{\textbf{Square-Sign}}
\def\pmair{$\mathbf{PM_{\text{2.5}}}$}
\def\adni{\textbf{ADNI}}
\def\clocks{\textbf{Clocks-LGS}}

\def\hd{\hat{d}}

\def\htheta{\hat{\theta}}
\def\hbeta{\hat{\beta}}
\def\hA{\hat{A}}

\def\hbK{\hat{\bf K}}

\def\hh{\hat{h}}
\def\bK{{\bf K}}
\def\balpha{\boldsymbol{\alpha}}
\def\hPhi{\hat{\Phi}}
\def\hSigma{\hat{\Sigma}}
\def\hZ{\hat{Z}}
\def\hbZ{\hat{\bf Z}}
\def\hcZ{\hat{\mathcal{Z}}}

\def\olR{\overline{R}}

\def\bK{{\bf K}}

\def\bR{{\bf R}}

\def\bX{\mathbf{X}}
\def\bY{\mathbf{Y}}
\def\bZ{{\bf Z}}

\def\bbN{\mathbb{N}}

\def\bbR{\mathbb{R}}

\def\cD{\mathcal{D}}

\def\cF{\mathcal{F}}

\def\cH{\mathcal{H}}

\def\cL{\mathcal{L}}

\def\cN{\mathcal{N}}

\def\cU{\mathcal{U}}

\def\cX{\mathcal{X}}
\def\cY{\mathcal{Y}}
\def\cZ{\mathcal{Z}}

\title{Efficient learning of nonlinear prediction models with time-series privileged information}

\author{%
  Bastian Jung \\
  Chalmers University of Technology \\
  \texttt{mail@bastianjung.com} \\
  \And
  Fredrik D. Johansson \\
  Chalmers University of Technology \\
  \texttt{fredrik.johansson@chalmers.se} \\
}

\begin{document}

\maketitle

\begin{abstract}
In domains where sample sizes are limited, efficient learning algorithms are critical. Learning using privileged information (LuPI) offers increased sample efficiency by allowing prediction models access to auxiliary information at training time which is unavailable when the models are used. In recent work, it was shown that for prediction in linear-Gaussian dynamical systems, a LuPI learner with access to intermediate time series data is never worse and often better in expectation than any unbiased classical learner. We provide new insights into this analysis and generalize it to nonlinear prediction tasks in latent dynamical systems, extending theoretical guarantees to the case where the map connecting latent variables and observations is known up to a linear transform. In addition, we propose algorithms based on random features and representation learning for the case when this map is unknown. A suite of empirical results confirm theoretical findings and show the potential of using privileged time-series information in nonlinear prediction.
\end{abstract}

%
%
\section{Introduction}
\label{sec:introduction}
In data-poor domains, making the best use of all available information is central to efficient and effective machine learning. Despite this, supervised learning is often applied in such a way that informative data is ignored. A good example of this is learning to predict the condition of a patient at a set follow-up time based on information of the first medical examination.
Classical supervised learning makes use only of the initial data to predict the disease status at the follow-up, although in many cases data about medications, laboratory tests or vital signs are routinely collected about patients also at intermediate time points.
This information is \emph{privileged}~\citep{vapnik2009new}, as it is unavailable at the time of prediction, but can be used for training a model.

\emph{Learning using privileged information} (LuPI)~\citep{vapnik2009new}, \emph{generalized distillation}~\citep{LopSchBotVap16} and \emph{multi-view learning}~\citep{rosenberg2007rademacher} have been proposed to increase the sample efficiency by leveraging privileged information in learning. Theoretical results guarantee improved learning rates~\citep{pechyony2010theory,vapnik2015learning} or tighter generalization bounds~\citep{wang2019everything} for large sample sizes under appropriate assumptions. However, privileged information is not always beneficial; it must be related to the task at hand~\citep{jonschkowski2015patterns}. Previous works do not identify such settings. Moreover, existing analyses do not state when learning with privileged information is preferable to classical learning for problems with small sample sizes---which is where efficiency is needed the most.

\citet{karlsson2021using} studied LuPI in the context of predicting an outcome observed at the end of a time series based on variables collected at the first time step. They showed that making use of data from intermediate time steps in particular settings always leads to lower or equal prediction risk---for any sample size---compared to the best unbiased model which does not make use of this privileged information. However, their method  called \emph{learning using privileged time series} (LuPTS) was limited to settings where the outcome function, and estimators of it, are linear functions of baseline features. Moreover, their analysis did not study how variance reduction behaves as a function of increased input dimension. \cite{hayashi2019long} also learned from privileged intermediate time points but their study was limited to empirical results for classification using generalized distillation.

There is an abundance of real-world prediction tasks with fixed follow-up times which can be framed as having access to privileged time-series information. Examples include predicting 90-day patient mortality \citep{karhade2019predicting} or patient readmission in healthcare~\citep{mortazavi2016analysis}, the churn of users of an online service over a fixed period~\citep{huang2012customer} or yearly crop yields from satellite imagery of farms \citep{you2017deep}. In these cases, privileged time-series information comprises daily patient vitals, intermediate user interactions and daily satellite imagery respectively. We are motivated by finding sample-efficient learning algorithms that utilize privileged time-series information to improve upon classical learning alternatives in such settings.

\paragraph{Contributions.} We extend the LuPTS framework to nonlinear models and prediction tasks in latent dynamical systems (Section~\ref{sec:models}). In this setting, we prove that learning with privileged information leads to lower risk when the nonlinear map connecting latent variables and observations is known up to a linear transform (Section~\ref{sec:fixed_rep}). In doing so, we also find that when the  representation dimension grows larger than the number of samples, the benefit of privileged information vanishes. We show that a privileged learner using random feature maps can learn optimal models consistently, even when the relationship between latent and observed variables is unknown, and give a practical algorithm based on this idea (Section~\ref{sec:kernels_and_maps}). However, random feature methods may suffer from bias in small samples. As a remedy, we propose several representation learning algorithms aimed at trading off bias and variance (Section~\ref{sec:learn_rep}). In experiments, we find that privileged time-series learners with either random features or representation learning reduce variance and improve latent state recovery in small-sample settings (Section~\ref{sec:experiments}) on both synthetic and real-world regression tasks.

%
%
\section{Prediction and privileged information in nonlinear time series}
\label{sec:background}%
\begin{figure}
    \centering
    \begin{subfigure}{.42\textwidth}
        \centering
        \vspace{0.35em}
        \includegraphics[width=.9\textwidth]{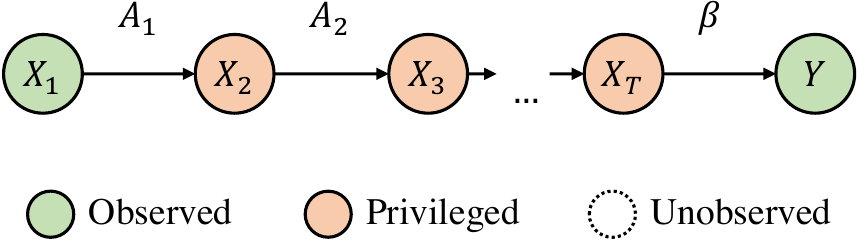}
        \vspace{0.35em}
        \caption{Linear dynamical system, fully observed\label{fig:obs_dgp}}
    \end{subfigure}
    \qquad
    \begin{subfigure}{.47\textwidth}
        \centering
        \includegraphics[width=.9\textwidth]{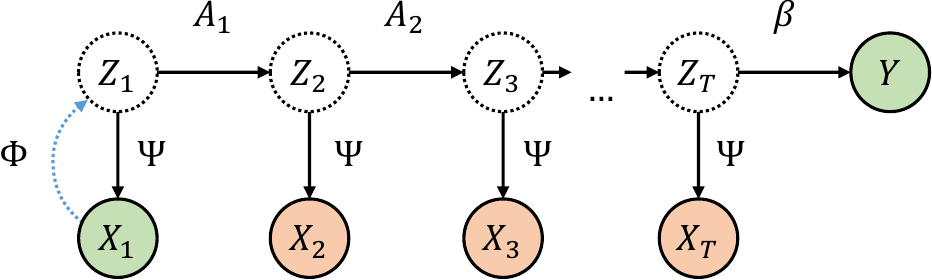}
        \caption{Linear latent dynamics, nonlinear observations\label{fig:latent_dgp}}
    \end{subfigure}
    \caption{Comparison between the linear data generating process and the latent nonlinear generalization in this work. $\Psi$ indicates the observation function of the latent system, with $\Phi$ its left inverse.}
    \label{fig:dgps}
    \vspace{-.5em}
\end{figure}

We aim to predict outcomes $Y$ in $\cY \subseteq \bbR^q$  based on covariates $X_1$ in $\cX \subseteq \bbR^k$. $X_1$ are observed at a baseline time point $t=1$, starting a discrete time series on the form $X_1, X_2, \dots X_T, Y$. Outcomes are assumed to be compositions of a representation $\Phi$, a linear map $\theta$ and  Gaussian noise $\epsilon$,
\begin{equation}\label{eq:nonliny}
Y = h(X_1) +\epsilon,\;\;\mbox{ where }\;\; h(x_1) \coloneqq \theta^\top \Phi(x_1),\;\;\mbox{ and }\;\;\epsilon \sim \cN(0, \tilde{\sigma}_Y^2)~.
\end{equation}
In addition to observations of $X_1$ and $Y$, \emph{privileged information} (PI) is available in the form of samples of random variables, $X_2, ..., X_T$, from intermediate time points between $X_1$ and $Y$, all taking values in $\cX$. Two data generating processes (DGPs) with this structure are illustrated in Figure~\ref{fig:dgps}. Unlike baseline variables $X_1$, privileged information is observed \emph{only at training time, not at test time}. Therefore, it can only benefit learning, and not inference.
Data sets of $m \in \bbN_+$ training samples, $D \coloneqq \{(x_{i,1}, x_{i,2}, ..., x_{i,T}, y_i)\}_{i=1}^m$, are drawn independently and identically distributed from a fixed, unknown distribution $p$ over all random variables in our system. We let $\bX_t = [x_{1,t}, ..., x_{m,t}]^\top$ denote the data matrix of features observed at time $t=1, ..., T$ and $\bY = [y_{1}, ..., y_{m}]^\top$ the vector of all outcomes observed in $D$.
A learning algorithm $\mathscr{A} : \cD \rightarrow \cH$ maps data sets $D \in \cD$ to hypotheses $\hh \in \cH$.
An \emph{efficient} algorithm $\scrA$ has small expected risk $\olR_p(\scrA)$ with respect to a loss $L : \cY \times \cY \rightarrow \bbR$ over training sets of fixed size $m$ drawn from $p$,
\begin{equation}\label{eq:exprisk}
\olR_p(\mathscr{A}) \coloneqq \E_D[R_p(\hh)] \;\;\mbox{ where }\;\; \hh = \mathscr{A}(D) \;\;\mbox{ and }\;\; R_p(\hh) \coloneqq \E_p[L(\hh(X_1), Y)]~.
\end{equation}
We study the regression setting with $L$ the squared error, $L(y,y') = \|y-y'\|_2^2$. In analysis, we focus on univariate outcomes, $q=1$, but all results extend to multivariate outcomes, $q>1$.

Our main goal is to compare \emph{privileged learners} $\scrA_\PI$, which make use of privileged information, to \emph{classical learners} $\scrA_\CL$ which learn without PI. We seek to identify conditions and algorithms for which privileged learning is more efficient, i.e., it leads to lower risk in expectation, $\olR(\mathscr{A}_\PI) \leq \olR(\mathscr{A}_\CL)$ for the same number of training samples. We describe such a setting next.

\subsection{Privileged information in latent dynamical systems}

We study tasks where data are produced by a latent dynamical process. Observations $X_t$ are generated from unobserved latent states $Z_t$ through a nonlinear transformation $\Psi$, see Figure~\ref{fig:latent_dgp}. This means that the target of learning, $h(X_1) = \E[Y\mid X_1]$, is nonlinear in $X_1$.
Latent dynamical systems like these have proven successful at modelling a variety of phenomena as for example fluid dynamics in physics \citep{lee_latent_dynamics} and human brain activity in neuroscience \citep{Kao2015}.

\begin{thmasmp}[Latent linear-Gaussian system]\label{asmp:linlatent}%
Let $Z_1, ..., Z_T$ be latent states related as a  linear-Gaussian dynamical system in a space $\cZ \subseteq \bbR^d$ with $Z_1$ of arbitrary distribution. Further, let the observation function be an injective map $\Psi : \cZ \rightarrow \cX$ with left-inverse (representation) $\Phi : \cX \rightarrow \cZ$. With $A_1, ..., A_{T-1} \in \bbR^{d \times d}$, $\beta \in \bbR^d$ assume that $Z_1, ..., Z_T, X_1, ..., X_T, Y$ are generated as
\begin{align*}
Z_t & = {A_{t-1}^\top} Z_{t-1} + \epsilon_t \;\; \mbox{ for } \;\;  t\geq2,  \;\;\; \mbox{ and } \;\;\;
\forall t : X_t = \Psi(Z_t) \;\;\;\mbox{ and }\;\;\; Y = \beta^\top Z_T + \epsilon_Y~.%
\end{align*}
where $\epsilon_t\sim \cN(0, \sigma_t^2I)$ and $\epsilon_Y \sim \cN(0, \sigma_Y^2)$.
\end{thmasmp}
It is easy to verify that Assumption~\ref{asmp:linlatent} is consistent with \eqref{eq:nonliny}, but stronger: there are more systems that map $X_1$ to $Y$ as $\theta^\top \Phi(X_1)$ than systems where additionally $X_t = \Psi(Z_t)$. Nevertheless, it is much more general than the  results of \cite{karlsson2021using}, limited to the linear setting, i.e. $\Phi(x_1) = x_1$.

Next, we present a generalized version of the LuPTS algorithm of \cite{karlsson2021using}  that is provably preferable to classical learning when data is generated according to Assumption~\ref{asmp:linlatent} and $\Phi$ is known up to a linear transform. Further, we discuss how a privileged learner can be made universally consistent for unknown representations $\Phi$ when combined with random features.

%
%
\section{Efficient learning with time-series privileged information}
\label{sec:models}
We analyze and compare learners from the privileged and classical paradigms which produce estimates of the form $\hat{h}(x_1) = \htheta^\top \hPhi(x_1)$, as motivated by \eqref{eq:nonliny}. We let each use representations $\hPhi : \cX \rightarrow \hat{\cZ} \subseteq \bbR^{\hd}$ from a family $\cF$ \emph{shared by both paradigms}, so that the hypothesis class $\cH \ni \hat{h}$ is shared as well.

\begin{figure}[t!]
\noindent
\begin{minipage}[t]{.55\textwidth}
\SetKwInput{kwParam}{Parameters}
\SetKwInput{kwInput}{Input}
\begin{algorithm}[H]
    \caption{Generalized LuPTS \label{alg:genlupts}}
    \kwInput{Data $D=(\{\bX_t\}, \bY)$; Repr. $\hPhi$ or kernel $\kappa$}
    \uIf{\textnormal{using a fixed representation} $\hPhi$}{
         $\hbZ_t = [\hPhi(x_{1, t}), ..., \hPhi(x_{m, t})]^\top$ for $t=1, ..., T$

         $\htheta_\PI \coloneqq \big[\prod\limits_{t=1}^{T-1} \underbrace{(\hbZ_t^\top\hbZ_t)^\dagger \hbZ_t^\top \hbZ_{t+1}}_{\hA_t}\big] \underbrace{(\hbZ_T^\top\hbZ_T)^\dagger \hbZ_T^\top \bY}_{\hbeta}$

         $\hat{h}_\PI(\cdot) \coloneqq \htheta_\PI^\top \hPhi(\cdot)$
    }
    \vspace{.5em}
    \uElseIf{\textnormal{using kernel} $\kappa$}{
        $\hbK_t = \kappa(\bX_t, \bX_t)$ for $t=1, ..., T$

        $\balpha \coloneqq \hbK_1^\dagger \left[ \prod_{t=2}^{T} \hbK_t \hbK_t^\dagger \right]\bY$

        $\hat{h}_\PI(\cdot) \coloneqq \sum_{i=1}^m \alpha_{i} \kappa(x_{i,1}, \cdot)$
    }
    \Return $\hat{h}_\PI$
\end{algorithm}
\end{minipage}%
\hfill
\begin{minipage}[t]{.44\textwidth}
\centering
\vspace{.01pt}
\includegraphics[width=0.92\textwidth]{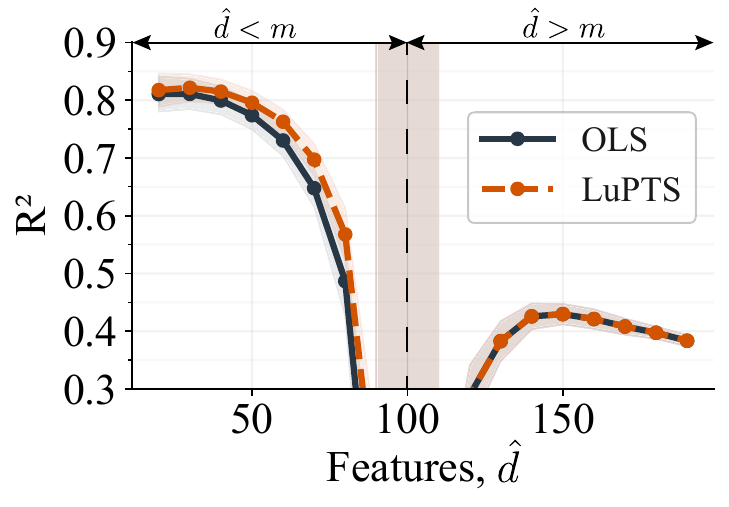}%
\vspace{-.5em}
\captionof{figure}{Two regimes of generalized LuPTS with varying feature dimension $\hat{d}$, and sample size $m=100$. When $m > \hat{d}$, generalized LuPTS is provably never worse than OLS. When $m \leq \hat{d}$, they are equivalent. See Appendix~\ref{app:additional_experiments} for details.}\label{fig:limitations}
\end{minipage}
\end{figure}

\subsection{Learning with true representations known up to a linear transform}
\label{sec:fixed_rep}
In this section, we assume that privileged and classical learners have access to a common, fixed representation function $\hPhi$ which is related to the true representation function through a linear transform, meaning there exists a matrix $B$ such that $\hPhi(\cdot)=B\Phi(\cdot)$.
\paragraph{Classical learning.} As a classical learner to serve as a strong comparison point for our privileged learners, we use the ordinary least-squares (OLS) linear estimator applied to the latent variables at the first time point inferred by $\hPhi$. With $\hbZ_1 = [\hPhi(x_{1, 1}), ..., \hPhi(x_{m, 1})]^\top \in \mathbb{R}^{m \times \hat{d}}$,
\begin{equation} \label{eq:classical_estimate}
\scrA_\CL(D) = \hat{h}_\CL(\cdot) \coloneqq  \htheta_\CL^\top\hPhi(\cdot) \;\;\mbox{ where }\;\; \htheta_\CL \coloneqq  (\hbZ_1^\top\hbZ_1)^\dagger\hbZ_1^\top\bY~.
\end{equation}
When $m \geq \hd$ and $\Phi$ is known up to linear transform, 
$\hh_\CL$ is the minimum-risk unbiased estimator of $h$ which \emph{does not} use privileged information.
To accommodate the underdetermined case where $m < \hd$, the matrix inverse is replaced by the Moore-Penrose pseudo inverse $(\cdot)^\dagger$~\citep{penrose1956best}.

\paragraph{Generalized LuPTS.} For privileged learning, we first compute $\hbZ_t = [\hPhi(x_{1, t}), ..., \hPhi(x_{m, t})]^\top$ for $t=1, ..., T$. Then, we independently fit  parameter estimates $\hA_1, ..., \hA_T, \hbeta$ of the dynamical system shown in Figure~\ref{fig:latent_dgp} by minimizing the squared error in single-step predictions. This equates to a series of OLS estimates in $\hcZ$. 
At test time, baseline variables $x_1$ are embedded with $\hPhi$ and the latent dynamical system is simulated for $T$ time steps to form a prediction $\hh_\PI(x_1) = (\hA_1 \cdots \hA_T\beta)^\top \hPhi(x_1)$. Putting this together, we arrive at Algorithm~\ref{alg:genlupts} which we call \emph{generalized LuPTS}. We may apply a simple matrix identity and replace terms $\hbZ_t \hbZ_t^\top$ in Algorithm~\ref{alg:genlupts} by the Gram matrices $\hbK_t = \kappa(\bX_t, \bX_t)$ of a reproducing kernel $\kappa$ with corresponding (implicit) feature map $\hPhi$, $\kappa(x,x') = \langle \hPhi(x), \hPhi(x')\rangle$~\citep{smola1998learning}. This variant allows for learning with unknown $\Phi$.

We now state a result which says that under Assumption~\ref{asmp:linlatent}, for an appropriate fixed representation $\hPhi$ or kernel $\kappa$, generalized LuPTS is never worse in expectation than the classical learner in \eqref{eq:classical_estimate}.

\begin{thmthm}\label{thm:main}%
Let $D$ be a data set of size $m$ drawn from $p$, consistent with Assumption~\ref{asmp:linlatent}. Assume that the left inverse $\Phi : \cX \rightarrow \cZ$ of the observation function $\Psi$ is known up to a linear transform, explicitly or through a kernel $\kappa(x,x') = \langle \hPhi(x), \hPhi(x')\rangle$, i.e., there exists a matrix $B$ with linearly independent columns such that $\hPhi(x) = B\Phi(x)$ with $\hPhi : \cX \rightarrow \hcZ$ for all $x \in \cX$. Then, it holds for the privileged learner $\scrA_\PI(D) = \hat{h}_\PI$ (Algorithm~\ref{alg:genlupts}) and the classical learner $\scrA_\CL(D) = \hat{h}_\CL$ \eqref{eq:classical_estimate},
\begin{equation}
\olR(\scrA_\PI) = \olR(\scrA_\CL) - \E_{\hat{h}_\PI, X_1}[\V_D(\hat{h}_\CL(X_1) \mid \hat{h}_\PI)]~.%
\end{equation}%
\end{thmthm}%
\begin{proofsketch}%
First, we show that the predictions made by generalized LuPTS are invariant to a linear transform $B$ applied to $\hPhi(\cdot)$ during training, see Appendix~\ref{app:proof}. Then we consider two cases: (i) When $\hat{d}\leq m$ we may re-purpose the proof of Theorem~1 in \cite{karlsson2021using}, (ii)~When $m < \hat{d}$ Proposition~\ref{prop:coincide} below directly implies $\E_{\hat{h}_\PI, X_1}[\V_D(\hat{h}_\CL(X_1) \mid \hat{h}_\PI)]=0$ and $\olR(\scrA_\PI) = \olR(\scrA_\CL)$.
\end{proofsketch}%
Theorem~\ref{thm:main} implies that the privileged learner is at least as sample efficient as the classical one since $\V(\cdot) \geq 0$ and thus $\olR(\scrA_\PI) \leq \olR(\scrA_\CL)$ for the same number of training samples $m$. The result is a direct generalization of the main result in \citet{karlsson2021using}. We also observe that generalized LuPTS and the classical learner coincide under certain conditions when $m \leq \hat{d}$.
\begin{thmprop}\label{prop:coincide}
Let $\hPhi : \cX \rightarrow \hcZ \subseteq \bbR^{\hd}$ be any map with corresponding kernel $\kappa$. Let $\hbK_t$ be the Gram matrix of $\kappa$ applied to $\bX_t$ and let $\hat{h}_\CL, \hat{h}_\PI$ be classical \eqref{eq:classical_estimate} and privileged (Algorithm~\ref{alg:genlupts}) estimates. Then,
$$
\hbK_t \;\mbox{is invertible for all}\; t \implies \hat{h}_\PI = \hat{h}_\CL.
$$
$\hbK_t$ is noninvertible whenever $m > \hat{d}$, assuming linearly independent features.%
\end{thmprop}%
\begin{proofsketch}%
When the Gram matrices $\bK_t$ are invertible, the pseudo-inverse coincides with the inverse, and factors $\bK_t \bK_t^{\dagger}$ in the LuPTS estimators cancel, making LuPTS and CL equal.
\end{proofsketch}%
\begin{thmrems*}
Theorem~\ref{thm:main} extends the applicability of LuPTS to a) nonlinear prediction through a fixed feature map or b) kernel estimation and c) to the underdetermined case of $m < \hd$. While previous work is restricted to observed linear systems our result considers the case of latent linear dynamics which only need to be identified up to a linear transform (see Figure~\ref{fig:dgps}).
Proposition~\ref{prop:coincide} does not claim that no preferable privileged learner exists; it is a statement only about generalized LuPTS. We may relate the result to the double descent characteristic previously observed for other linear estimators for fixed $m$ and varying $\hat{d}$ \citep{prehistory_double_descent}. After a phase transition around $\hat{d}=m$ LuPTS's variance reduces for a second time when it becomes equivalent to the classical learner (see Figure~\ref{fig:limitations}).
\end{thmrems*}

\subsection{Random feature maps for unknown representations}
\label{sec:kernels_and_maps}

When the true $\Phi$ is entirely unknown, as is often the case in practice, using a poor representation $\hPhi$ may yield biased results for both classical and privileged learners. A common solution in nonlinear prediction is to use a universal kernel, such as the Gaussian-RBF kernel. These have dense reproducing-kernel Hilbert spaces which allow approximation of any continuous function. However, universal  kernels also have positive-definite and thus invertible Gram matrices \citep{Hofmann_2008}, which according to Proposition~\ref{prop:coincide} eliminates any gain in sample efficiency of generalized LuPTS.

Instead, we combine our algorithm with an approximation of universal kernels---random feature maps. These methods project inputs $x$ onto $\hd$ features by a random linear map $W \in \bbR^{k \times \hd}$, and a nonlinear element-wise activation function. By choosing $\hd < m$, we can benefit from the function approximation properties of universal kernels (see discussion below) and the variance reduction of generalized LuPTS.
Popular random features include random Fourier features (RFF) \citep{rahimi_recht} and random ReLU features (RRF) \citep{random_relu},\footnote{For RFF, $[W_\cN]_{ij} \sim \cN(0,1)$, and for RRF, $[W_\cU]_{ij} \sim \cU(-1,1)$ and $b_i \sim \cU(0,2\pi)$.}
$$
\hPhi_{\mbox{\sc rff}}^\gamma(x) = \sqrt{2/\hd}\left[\cos(\sqrt{2\gamma} W_\cN
^\top x + b) \right]
\;\;\mbox{ and }\;\;
\hPhi_{\mbox{\sc rrf}}^\gamma (x)= f_+(\gamma W_\cU^\top [x; 1])~.
$$
where $f_+(z) = \max(0, z)$ is the rectifier (ReLU) function and $\gamma > 0$ is a bandwidth hyperparameter.

For large enough numbers of random features and training samples, any continuous function $h$ can be approximated up to arbitrary precision by a linear map $\omega$ applied to the random features, e.g., $\hat{h}(x) = \omega^\top \hPhi_{\mathrm{RRF}}^\gamma(x)$, see~\citet{sun2019approximation,rudi2017generalization}. Applying the same argument to the step-wise estimators of LuPTS we can justify using random features in Algorithm~\ref{alg:genlupts} by the following observation: \emph{Under appropriate assumptions, we can construct a privileged learner using random feature maps which is a universally consistent estimator of $h(x_1) = \E[Y|x_1]$.} The precise construction deviates somewhat from Algorithm~\ref{alg:genlupts}, but follows the same structure. We refer to Appendix~\ref{app:universality} for a precise statement. As universal consistency describes the asymptotic behaviour in the limit of infinite samples and random features, this offers only limited insight into the benefits of privileged information in small sample settings, where performance will be a bias-variance trade-off.

\paragraph{Variance reduction \& bias amplification.} Generalized LuPTS is only guaranteed lower variance compared to the classical estimator under Assumption~\ref{asmp:linlatent}, although our empirical results (Section~\ref{sec:experiments}) suggest this applies more widely.
When $\hPhi$ is a bad approximation of $\Phi$, generalized LuPTS may amplify bias, increasing with the number of privileged time points, compared to classical learning.
We show this theoretically in Appendix~\ref{app:compounding_bias} and also empirically in Appendix~\ref{app:additional_experiments}. Whether generalized LuPTS is still preferable to classical learning in terms of prediction risk appears to depend on the amount of bias that gets amplified. Our experiments imply the variance reduction mostly dominates when using random features, whereas this is not always the case for linear LuPTS. The phenomenon of bias amplification is familiar from e.g., model-based and model-free reinforcement learning~\citep{kober_RL}. As the bias with random features may still be high for small sample settings, we next present privileged representation learning algorithms to trade off bias and variance more efficiently.

\subsection{Privileged time-series representation learning}
\label{sec:learn_rep}
Up until now the representation $\hPhi$ was considered fixed, either because $\Phi$ was known up to a linear transform (explicitly or implicitly) or because of the use of random feature methods.
Generalized LuPTS (Algorithm~\ref{alg:genlupts}) produces minimizers $\{\hA_t\}, \hbeta$ of the following objective  for fixed $\hPhi$,
\begin{equation}\label{eq:SRL}
    \mathcal{L}_{\SRL}(\hPhi, \{\hA_t\}, \hbeta) := \frac{1}{NT}\sum\limits_{i=1}^N \bigg{[}\sum\limits_{t=1}^{T-1} \frac{1}{\hat{d}}\big{\Vert}\hA_t^\top\hPhi(x_{i,t}) - \hPhi(x_{i,t+1})\big{\Vert}_2^2 + \frac{1}{q}\big{\Vert}\hbeta^\top\hPhi(x_{i,T})-y_i\big{\Vert}_2^2 \bigg{]}~.%
\end{equation}
Objective~\eqref{eq:SRL} and the systems described by Assumption~\ref{asmp:linlatent} lend themselves to methods which also learn the representation $\Phi$ in addition to the latent dynamics $\{A_t, \beta\}$. Next, we present three algorithms which combine the ideas of generalized LuPTS with the expressiveness of deep representation learning. All learners 
use equivalent encoders to represent $\hPhi(\cdot)$ and linear layers to model the relations between the latent variables $\{Z_t\}$ and the outcome $Y$. The classical learner predicts the outcome linearly from $\hat{Z}_1$. All architectures under consideration are visualized jointly in Figure~\ref{fig:rep_learners}.

\begin{figure}
    \centering
    \begin{subfigure}{.7\textwidth}
        \centering
            \includegraphics[width=0.93\textwidth]{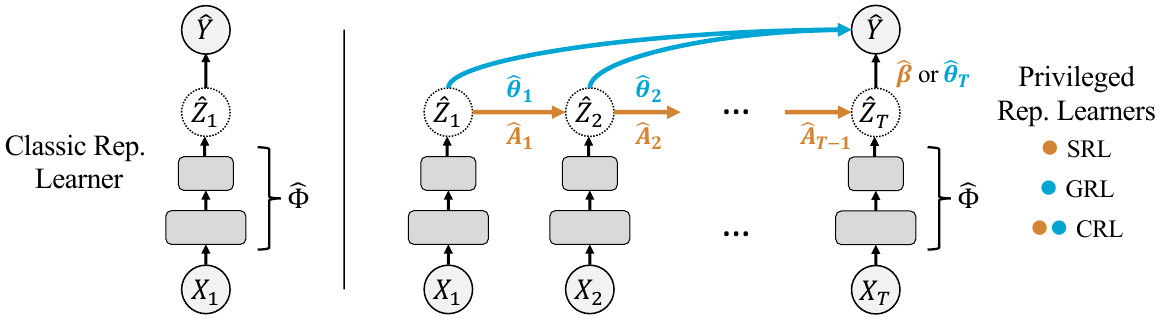}
        \caption{Classical (left) and privileged (right) representation learners. $\hPhi$ is shared across all time steps. GRL models the direct maps $\htheta_t$ to the outcome; SRL models the single steps $\hA_t$ and $\hbeta$. CRL combines the two.}
        \label{fig:rep_learners}
    \end{subfigure}
    \hfill
    \begin{subfigure}{.28\textwidth}
    \centering
        \includegraphics[width=.75\textwidth]{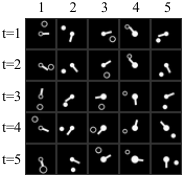}
        \caption{Five example sequences from \clocks{} image data. }
        \label{fig:clock_data_examples}
    \end{subfigure}
    \caption{Representation learning architectures (left) and samples from \clocks{} (right).}
    \vspace{-0.5em}
\end{figure}

\paragraph{SRL.}
The first privileged representation learner directly optimizes objective \eqref{eq:SRL}, just like generalized LuPTS, but now also fitting the representation $\hPhi$, parameterized by a neural network. We refer to this model as \emph{stepwise representation learner} (SRL).
As we will see in experiments, a drawback of this approach is that representations may favor predicting transitions $\hat{z}_{i,t}\rightarrow \hat{z}_{i,t+1}$ with small error, while losing information relevant for the target outcome in the process. At test time, for a new input $x_1$, SRL composes the stepwise dynamics to output   $\hat{h}_{\PI}(x_1)=\hat{\beta}^\top\hat{A}_T^\top\dots\hat{A}_1^\top\hPhi(x_1)$.

\paragraph{CRL and GRL.}
To make sure that the learned representation $\hPhi$ retains information about  $Y$, we add linear outcome supervision to the representation $\hPhi(X_t)$ at each time step $t$. Recall that, by Assumption~\ref{asmp:linlatent}, the expected outcome is linear in the latent state at \emph{any} time step. We introduce a hyperparameter $\lambda \in [0,1]$ to trade off the two types of losses and arrive at the \emph{combined representation learner} (CRL). With $\alpha = (\hPhi, \{\hA_t\}, \{\htheta_t\})$ the entire parameter vector, CRL minimizes the objective
\begin{equation}\label{eq:CRL}
    \mathcal{L}_\CRL(\alpha) := \frac{\lambda}{NTq}\sum\limits_{i,t}^{} \big{\Vert}\htheta_t^\top\hPhi(x_{i,t}) - y_i\big{\Vert}_2^2+
    \frac{1-\lambda}{N(T-1)\hd}\sum\limits_{i,t}^{} \big{\Vert}\hA_t^\top\hPhi(x_{i,t}) - \hPhi(x_{i,t+1})\big{\Vert}_2^2.
\end{equation}
We make test-time predictions using $\hat{h}_\PI(x_1)=\htheta_1^\top\hPhi(x_1)$. In experiments, we highlight the case $\lambda=1$ where only outcome supervision is used as \emph{greedy representation learner} (GRL). For a precise definition of the GRL objective, see Appendix~\ref{app:GRL}. GRL is related to multi-view learning, in which prediction of the same quantity is made from multiple ``views'' (cf. time points)~\citep{multi-view_survey}.

\section{Experiments}
\label{sec:experiments}
We compare classical learning to variants of generalized LuPTS (Algorithm~\ref{alg:genlupts}) and the privileged representation learners of Section~\ref{sec:learn_rep} on two synthetic and two real-world time-series data sets.
We (i) verify our theoretical findings by analyzing the sample efficiency and bias-variance characteristics of the given algorithms; (ii) demonstrate that generalized LuPTS with random features succeeds in settings where linear LuPTS suffers from large bias; (iii) point out that privileged representation learners offer even greater sample efficiency in practice and (iv) study how well these algorithms recover the true latent variables $\{Z_t\}$ and how this relates to  predictive accuracy.

\paragraph{Experimental setup.}
We report the mean coefficient of determination ($R^2$), proportional to the squared-error risk $\olR$, for varying sample sizes, sequence lengths and prediction horizons. Experiments are repeated and averaged over different random seeds. In each repetition, a given model performs hyperparameter tuning on the training data using random search and five-fold cross-validation before being re-trained on all training data.
The test set size is 1000 samples for synthetic data and 20\% of all data available for real-world data sets.
We consider six privileged learners of two groups. The first group comprises generalized LuPTS with the linear kernel (LuPTS) and the two random feature maps shown in Section~\ref{sec:kernels_and_maps}: Random Fourier features (Fourier RF) and random ReLU features (ReLU RF). The classical learners for this group are OLS estimators used with the same kernel or feature map.
The second group consists of the representation learners SRL, CRL and GRL. For tabular data, their encoder is a multi-layer perceptron with three hidden layers of 25 neurons each. For the image data they use LeNet-5 \citep{LeNet5}. The classical learner (Classic Rep.) uses the same encoder with a linear output layer.
The results presented were found to be robust to small changes in training parameters such as learning rate. For details on the training process we refer to Appendix~\ref{app:experiments_and_processing}.
All experiments required less than 3000 GPU-h to complete using NVIDIA Tesla T4 GPUs.\footnote{Code to reproduce all results is available at \href{https://github.com/Healthy-AI/glupts}{https://github.com/Healthy-AI/glupts}.}

\paragraph{Data sets.}
\label{sec:data_sets}
We briefly describe evaluation data sets and refer to Appendix \ref{app:experiments_and_processing} for further details. First, we create two synthetic data sets in which latent states and outcomes are generated from linear-Gaussian systems as in Assumption~\ref{asmp:linlatent}. To produce the observations $\{X_t\}$ we use a deterministic nonlinear function $\Psi:\cZ \xrightarrow[]{} \cX$. In the first synthetic data set, which we call \textbf{Square-Sign}, the nonlinear transformation $\Psi:\bbR^d \xrightarrow[]{} \bbR^{2d}$ maps each latent feature $Z_{(t,k)}$ to a two dimensional vector such that
\begin{equation}
    X_t \coloneqq \Psi(Z_{(t)})= [Z_{(t,1)}^2, \mathrm{sgn}(Z_{(t,1)}), \dots , Z_{(t,d)}^2, \mathrm{sgn}(Z_{(t,d)})]^\top.
    \nonumber
\end{equation}
The second synthetic data set uses the same latent linear system with $Z_t \in \bbR^2$ and produces square images (28$\times$28 pixels) as observations. As the images are reminiscent of clocks we refer to this data set as \clocks{}. Example sequences of these observations are presented in Figure~\ref{fig:clock_data_examples}. The angle, size and fill of the two clock hands encode the value of the corresponding latent variable. The outcome $Y$ is a linear function of $Z_T$ with $q=1$. For \squaresign{}, $q=3$.

The Metro Interstate traffic volume data set  (\traffic{})~\citep{MNTraffic} contains hourly records of the traffic volume on the interstate 94 between Minneapolis and St. Paul, MN. In addition, the data contains weather features and a holiday indication. We predict the traffic volume for a fixed time horizon given the present observations. Privileged information is observed every four hours.

We also predict the progression of Alzheimer's disease (AD) as measured by the outcome of the Mini Mental State Examination (MMSE)~\citep{MMSE}. The anonymized data were obtained through the Alzheimer's Disease Neuroimaging Initiative  (\adni{})~\citep{adni}. 
The outcome of interest is the MMSE score 48 months after the first examination. Privileged information are the measurements at 12, 24 and 36 months. In addition we tested our algorithms on the \pmair{} data set \citep{pm2.5}, where we predict the air quality in five Chinese cities (see Appendix~\ref{app:additional_experiments}).

\subsection{Sample efficiency, bias and variance}
\begin{figure}[t!]
\centering

\begin{subfigure}[t]{\textwidth}
\centering
    \includegraphics[width=\textwidth]{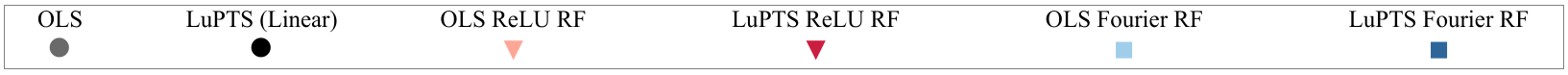}
\end{subfigure}

\begin{subfigure}[t]{0.3\textwidth}
        \includegraphics[width=\textwidth]{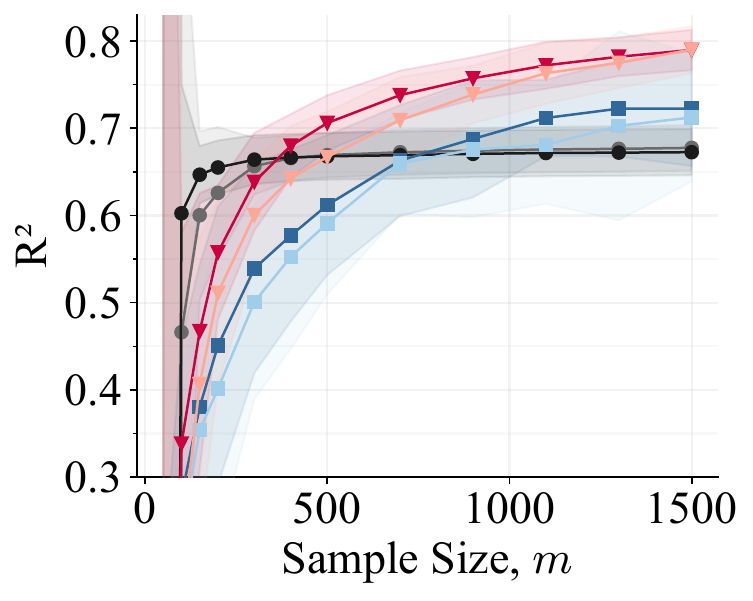}
        \caption{\traffic{}, $T=3$.}
        \label{fig:RF_Traffic_R2}
\end{subfigure}
\hfill%
\begin{subfigure}[t]{0.3\textwidth}
        \includegraphics[width=\textwidth]{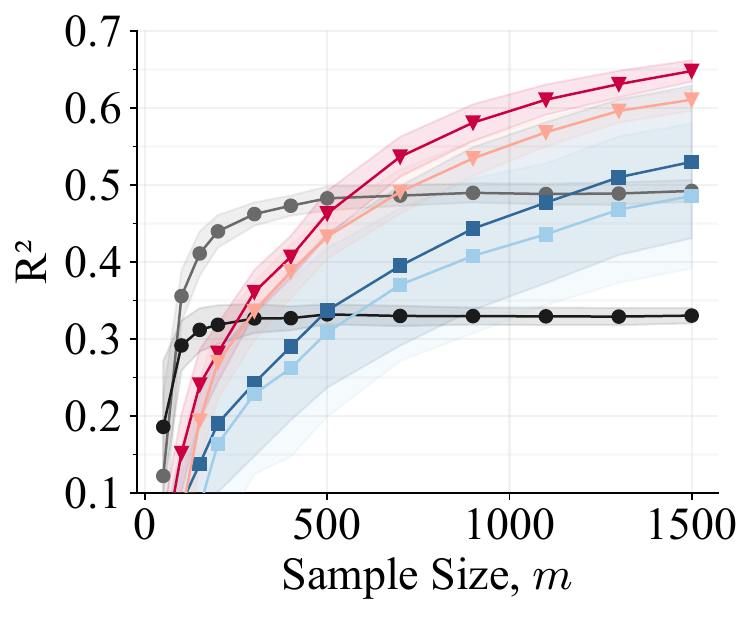}
        \caption{\squaresign{}, $T=5$, $d=10$.}
        \label{fig:RF_SS_R2}
\end{subfigure}
\hfill
\begin{subfigure}[t]{0.3\textwidth}
    \includegraphics[width=\textwidth]{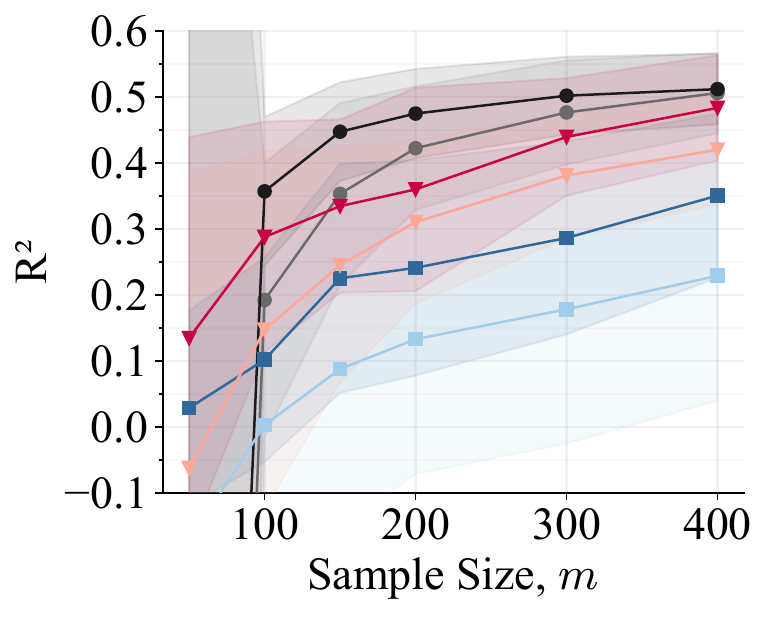}
    \caption{\adni{}, $T=4$.}
    \label{fig:RF_ADNI_R2}
\end{subfigure}
\caption{Predictive accuracy of generalized LuPTS on three data sets, over 60 repetitions. The shaded area represents one standard deviation above and below the mean over repetitions. }
\label{fig:RF_R2}
\end{figure}

\begin{figure}[t!]
    \centering
\begin{subfigure}[t]{\textwidth}
\centering
    \includegraphics[width=\textwidth]{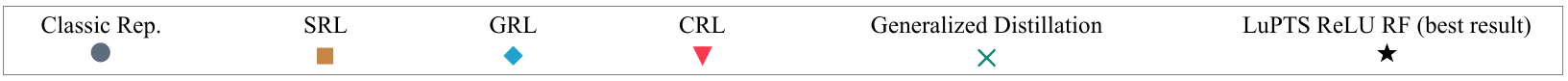}
\end{subfigure}
\begin{subfigure}[t]{0.3\textwidth}
        \includegraphics[width=\textwidth]{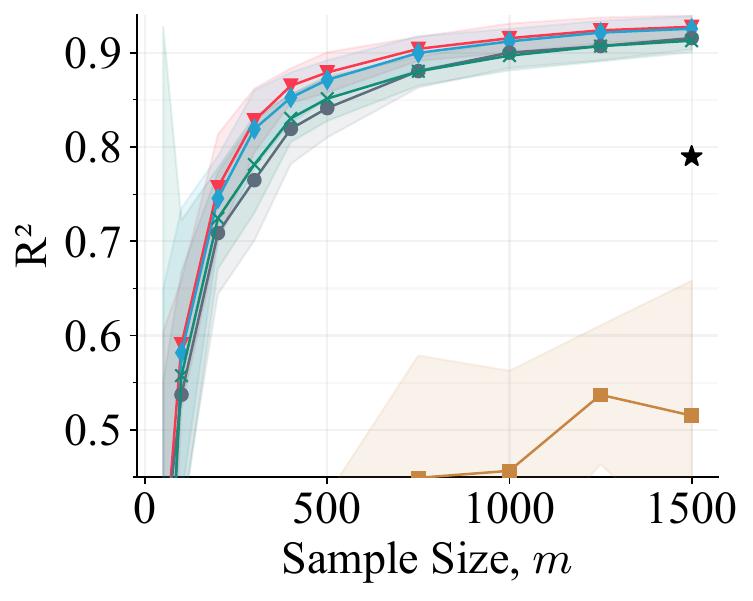}
        \caption{\traffic{}, $T=3$.}
        \label{fig:RL_TRAFFIC_R2}
\end{subfigure}
\hfill%
\begin{subfigure}[t]{0.3\textwidth}
        \includegraphics[width=\textwidth]{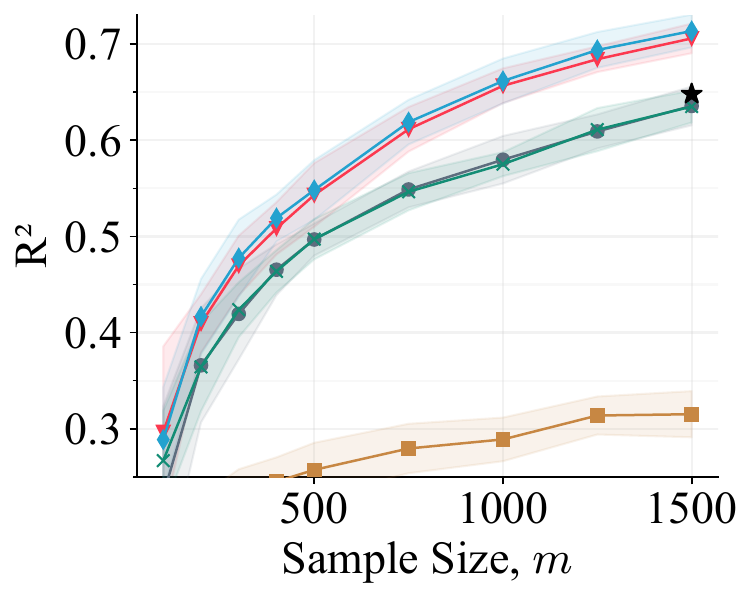}
        \caption{\squaresign{}, $T=5$, $d=10$.}
        \label{fig:RL_SS_R2}
\end{subfigure}
\hfill
\begin{subfigure}[t]{0.3\textwidth}
    \includegraphics[width=\textwidth]{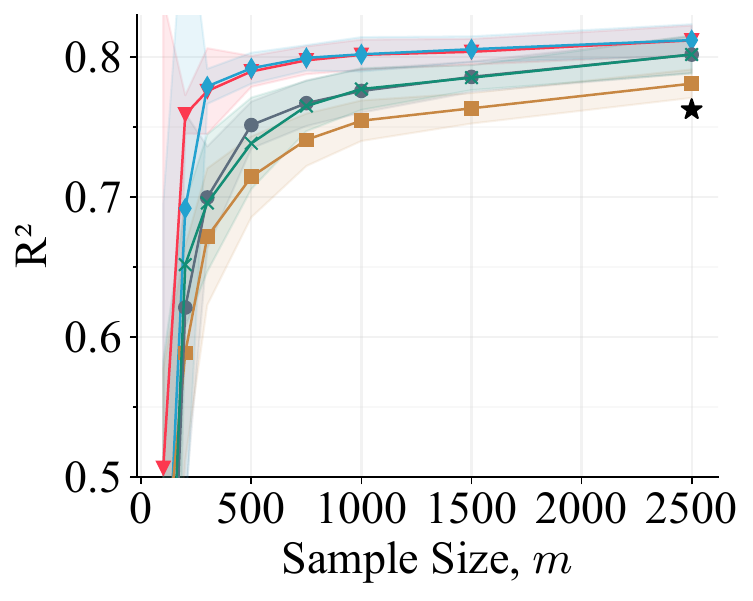}
    \caption{\clocks{}, $T=6$, $q=1$.}
    \label{fig:RL_CLOCK_R2}
\end{subfigure}
    \caption{Predictive accuracy of the representation learners over 25 repetitions. Generalized distillation is used as proposed by \citet{hayashi2019long}. For details, see Appendix~\ref{app:experiments_and_processing}.}
    \label{fig:RL_R2}
    \vspace{-1em}
\end{figure}

The main goal of our work is to improve learning efficiency by incorporating privileged time-series information. Across almost all prediction tasks and sample sizes, nonlinear variants of generalized LuPTS outperform their classical counterpart in terms of sample efficiency as can be seen in Figure~\ref{fig:RF_R2}. On the \traffic{} prediction task, LuPTS ReLU RF outperforms linear LuPTS as the former appears to exhibit less bias. On the synthetic data of \textbf{Square-Sign}, linear models reach their best accuracy quickly, while they are limited by their lack of expressiveness. Generalized LuPTS amplifies this bias, making linear LuPTS worse than OLS. Random feature methods attain higher accuracy here but are generally less sample efficient. Generalized LuPTS combined with random features manages to decrease this gap significantly.
On \adni{}, we don't see a benefit of using nonlinear models in general. We make similar observations on the \pmair{} air quality data set, see Appendix~\ref{app:additional_experiments}.

The representation learners proposed in Section~\ref{sec:learn_rep} are evaluated on the same data sets and on the image prediction task \clocks{}. Example results are displayed in Figure~\ref{fig:RL_R2}. These demonstrate that directly transferring the LuPTS objective to neural networks in the form of SRL results in subpar performance. 
SRL does not seem to have a strong enough incentive to learn representations which accurately predict the outcome. Generalized distillation for privileged time series as suggested by \cite{hayashi2019long}, does not improve upon classical learning in our tasks.
On all experiments displayed in Figure~\ref{fig:RL_R2}, CRL and GRL outperform the classical learner. The predictive accuracy of these models is similar on most tasks, as may be explained by the fact that CRL may reduce to GRL when choosing $\lambda=1$ in objective~\ref{eq:CRL}. Noticeably, the general observation of GRL and CRL being more sample efficient than the classic model, neither appears to depend on the neural architecture used for the encoder, nor does the modality of the data play an important role, as the image prediction task \clocks{} (Figure~\ref{fig:RL_CLOCK_R2}) demonstrates. For additional empirical results, we refer to Appendix~\ref{app:additional_experiments}.

\begin{figure}[t!]
    \centering
    \includegraphics[width=\textwidth]{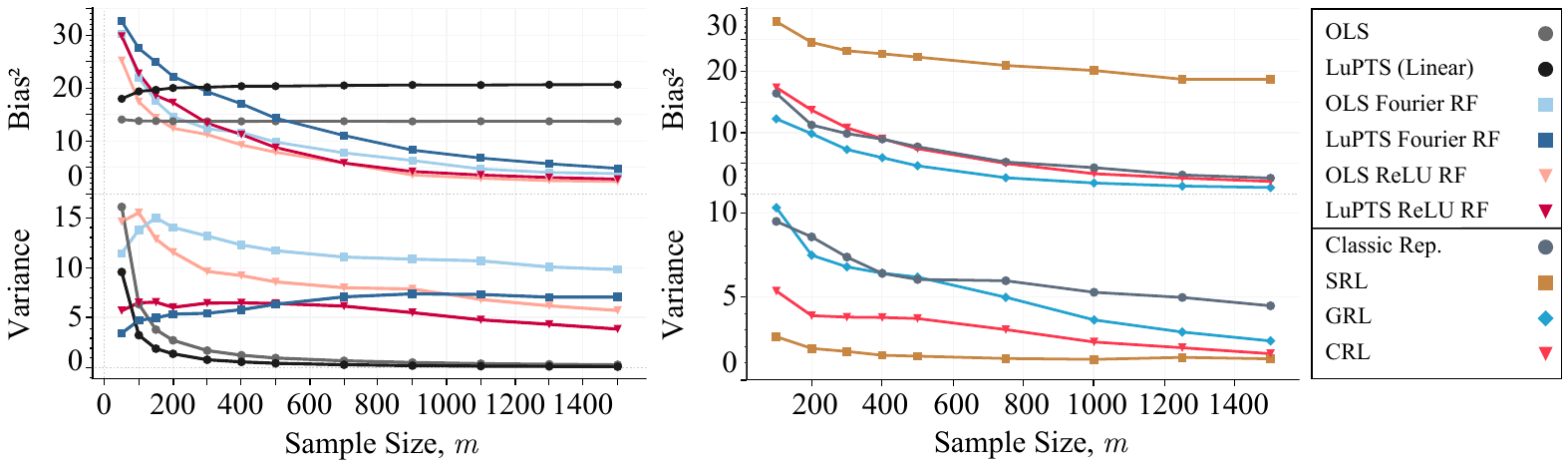}

    \caption{Estimated squared bias and variance of models trained on \squaresign{} ($T=5$, $d=10$, $q=3$) over random training sets (60 for left group, 25 for the right), evaluate on 1000 test points.}
    \label{fig:BIAS_VAR}
    \vspace{-1.em}
\end{figure}

To analyze the bias and variance characteristics of our algorithms, we can estimate the expected squared prediction \emph{bias}, $\E_{X_1}[(\E_D[\scrA(D)](X_1) - \E[Y|X_1])^2]$, by computing $\E_Y[Y|X_1]$ in synthetic DGPs in addition to the variance of the different estimators. 
Figure~\ref{fig:BIAS_VAR} depicts bias and variance for all models on the \squaresign{} data. On the left panel, all variants of generalized LuPTS exhibit lower variance than classical learning, despite being biased.
This holds generally: \emph{Across all experiments, we never encountered an example where the use of privileged information has not resulted in lower variance compared to classical learning.} On the contrary, the privileged learners suffer higher bias than the comparable classical learners because of the bias compounding over the individual prediction steps as shown in Appendix~\ref{app:compounding_bias}. For the random feature variants however, the bias decreases with the number of samples.
The representation learners display similar characteristics on the right panel of Figure~\ref{fig:BIAS_VAR}. 
Learning the transitions between latent variables $Z_t$ appears to be associated with low variance and high bias as demonstrated by the results of SRL. GRL however, which does not model these transitions, exhibits the lowest bias and the largest variance of all privileged learners. As CRL is able to trade off between these two objectives, the estimates for its variance and bias lie in between the corresponding values of the other two privileged learners.

\subsection{Latent variable recovery}
\label{sec:latent_var_recovery}

\begin{figure}[t!]
\begin{subfigure}[t]{0.39\textwidth}
    \includegraphics[width=1.\textwidth]{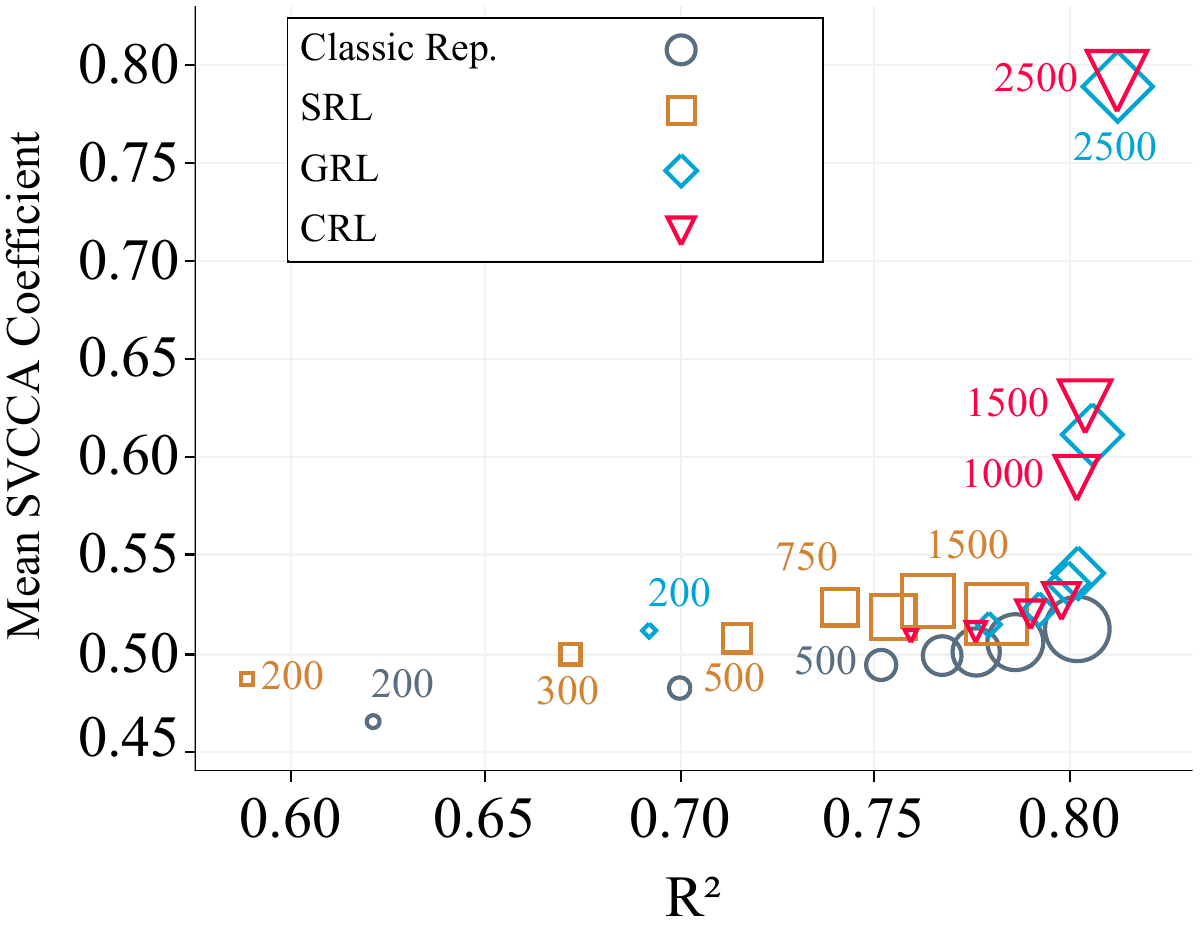}
    \caption{Mean $R^2$ and SVVCA coefficient over 25 training runs. Markers represent average results over repeated experiments with size indicating the sample size, $m$.}
    \label{fig:clock_svcca_scatter_y1}
\end{subfigure}
\hfill
\begin{subfigure}[t]{0.59\textwidth}
    \centering
    \includegraphics[width=\textwidth]{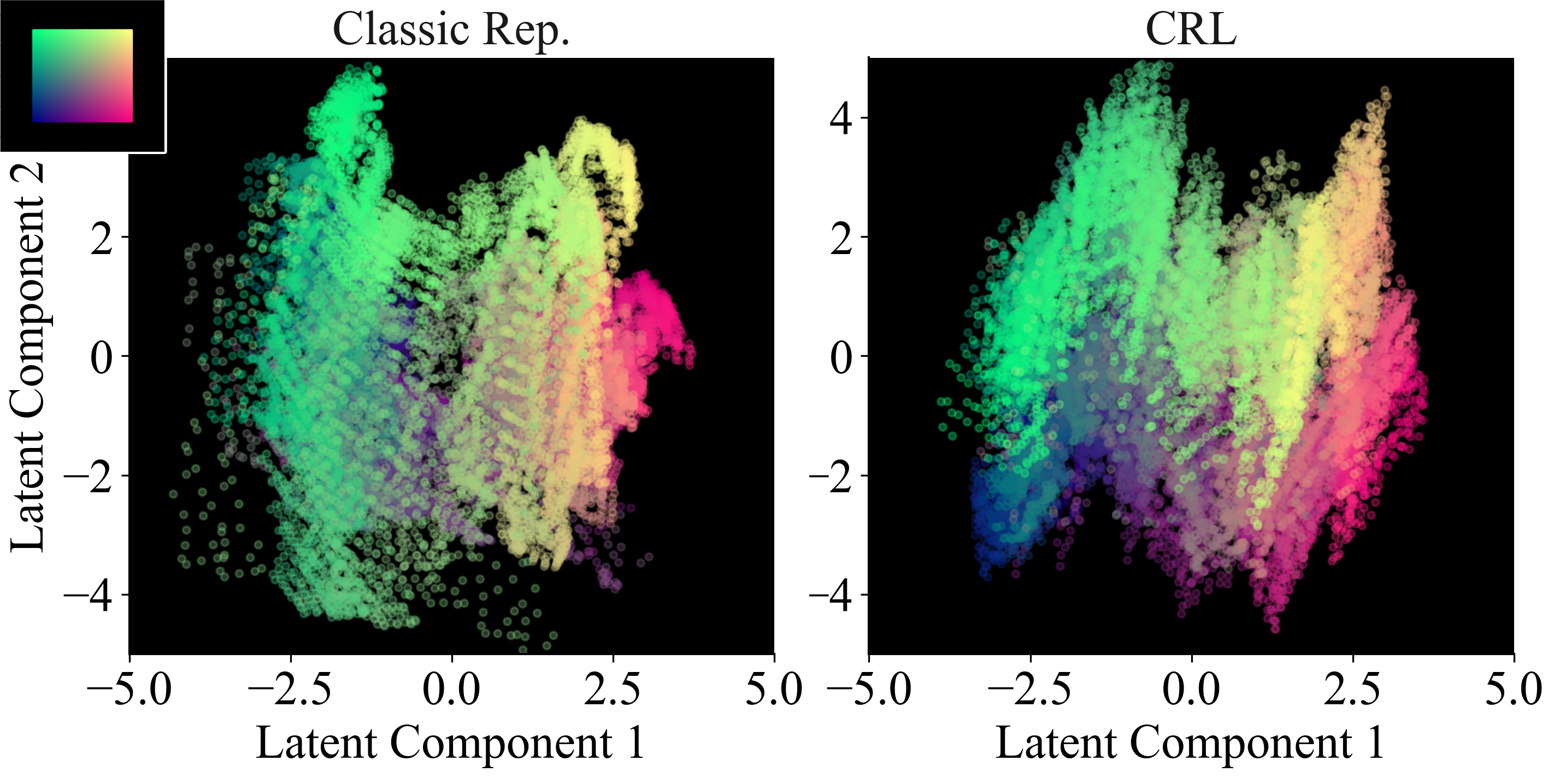}
    \caption{Visualizing the representations learned by Classic Rep. and CRL after applying SVCCA. The recovery target is shown in the top left corner. Both estimators are best in class and trained on 2500 samples.}
    \label{fig:svcca_rainbow_clock_y1}
\end{subfigure}
\caption{Analyzing the representations learned by the models of Section~\ref{sec:learn_rep} using SVVCA on the \clocks{} image regression data set (same setup as in Figure~\ref{fig:RL_CLOCK_R2}).}
\vspace{-1em}
\end{figure}
Under Assumption~\ref{asmp:linlatent}, it is sufficient to identify the representation function $\Phi$ up to a linear transform $B$ to have provable gains from privileged information over a classical learner. In synthetic data, we can assess to what extent a representation $\hPhi$ with this property has been found. As a proxy for the existence of such a transform, we can compute a measure of correlation between $\hPhi(X)$ and $\Phi(X)$, such as the Canonical Correlation Analysis (CCA)~\citep{cca_hotelling}. \cite{SVCCA} introduced a modified version called Singular Vector Canonical Correlation Analysis (SVCCA) to compute correlations when dealing with noisy dimensions in neural network representations.

The mean SVCCA coefficients $\Bar{\rho}$ and predictive accuracy of the representation learners are visualized in Figure~\ref{fig:clock_svcca_scatter_y1}. One notices that GRL and CRL produce higher correlation coefficients than the classical learner while also predicting the outcome on the \clocks{} task more accurately.
Further comparing the representations learned by privileged and classical models, Figure~\ref{fig:svcca_rainbow_clock_y1} shows a visual example of the improved latent recovery of CRL on the same task. Concluding, the use of privileged time-series information does not only increase the sample efficiency of existing algorithms but can also aid the recovery of latent variables in latent dynamical systems.

\section{Related work}
\label{sec:related}
Existing analyses for learning with privileged information guarantee improvements in asymptotic sample efficiency under strong assumptions~\citep{vapnik2009new} but are insufficient to establish a clear preference for LuPI learners for a fixed sample size. For example, \citet{pechyony2010theory} showed that for a specialized problem construction, empirical risk minimization (ERM) using privileged information can achieve fast learning rates, $\mathcal{O}(1/n)$, while classical (non-privileged) ERM can only achieve slow rates, $\mathcal{O}(1/\sqrt{n})$. We are not aware of any generalization theory tight enough to establish a lower bound on the risk of a classical learner larger than an upper bound for a PI learner. Our problem is also related to multi-task (representation) learning~\citep{maurer2016benefit}, see especially \eqref{eq:SRL}--\eqref{eq:CRL}. However, our goal is different in that only a single task is of interest after learning.

In estimation of causal effects, learning from \emph{surrogate outcomes}~\citep{prentice1989surrogate} has been proposed as a way to increase sample efficiency. Surrogates are variables related to the outcome which may be available even when the outcome is not~\citep{athey2019surrogate}. We can view these as privileged information. While the problem shares structure with ours, the goal is to compensate for \emph{missing} outcomes and analytical results give no guarantees for improved efficiency when both surrogates outcomes are always observed~\citep{chen2008improving,kallus2020role}. \citet{guo2022efficient} showed that, in the context of a linear-Gaussian system on a directed acyclic graph, a recursive least-squares estimator is the asymptotically most efficient estimator of causal effects using only the sample covariance. In the linear case on a path graph, their estimator coincides with ours. However, no analysis is provided for the nonlinear case or for fixed sample sizes.


\section{Discussion}
\label{sec:conclusion}

We have presented learning algorithms for predicting nonlinear outcomes by utilizing time-series privileged information during training. We prove that our estimator is preferable to classical learning when data is generated from a latent-variable dynamical system with partially known components. The proof holds for the case when the latent dynamical system is recovered by the representation function up to a linear transform, assuming that a left inverse of the true observation generating function exists. However, this assumption does not appear to be necessary, as our empirical results demonstrate that privileged learning is preferable to classical learning even when these conditions cannot be guaranteed. Consequently, a more general theoretical result where less is known about the latent system seems attainable. For example, one might consider a case in which only a few independent components of the latent variables are recovered by the representation used by the learning algorithm, while other components are treated as noise.

When the latent dynamical system is entirely unknown, we create practical estimators using random feature embeddings which outperform the corresponding classical learner across experiments. We show that a universally consistent learner can be constructed based on this idea, with slightly different form. As a further alternative, we propose representation learning methods of related form using neural networks and demonstrate the empirical benefits also of this estimator over classical learning. In experiments, we analyze how the gap in risk between privileged and classical learning changes for different prediction horizons, as displayed in Figure~\ref{fig:ss_svcca_scatter_app} in Appendix~\ref{app:additional_experiments}. The results suggest that the risk advantage of privileged learners grows with the sequence length of the prediction task, despite the fact that the task becomes more difficult at the same time.

Our work focuses on the setting where data is observed as a time series. This setting is chosen for its simple causal structure given by (latent) linear-Gaussian systems and because it can be motivated from many different applications. However, the ideas presented are not specifically tied to time and also apply in the case when all variables are observed simultaneously as long as the causal structure remains sequential. Moreover, we believe that the theory presented here is not limited to sequential settings and generalizes to other causal structures, in particular directed acyclic graphs that connect the baseline covariates to the outcome. In either case, one might only have access to the baseline variable at test time. For example, the timely collection of data for all covariates at test time might be very expensive or even impossible.

As pointed out before, Theorem~\ref{thm:main} requires the recovery of the latent variables up to a linear transformation. Nonlinear independent component analysis (ICA) \citep{hyvarinen2016unsupervised} aims to solve precisely this problem and has been applied to time series via time-contrastive learning. This makes for an interesting connection between learning using privileged information and nonlinear ICA, as the experiments of Section~\ref{sec:latent_var_recovery} suggest that privileged time-series information aids the recovery of latent variables.
Other remaining challenges include providing risk guarantees for learning with biased representations (including deep neural networks), with regularized estimators, and for more general data generating processes with weaker structural assumptions. We are hopeful that the utility demonstrated in this work will inspire future research to overcome these limitations.

\begin{ack}
We would like to thank Anton Matsson and Rickard Karlsson for insightful feedback and the Alzheimer’s Neuroimaging Initiative (ADNI) for collecting and providing the data on Alzheimer's disease used in this project. The present work was funded in part by the Wallenberg AI, Autonomous Systems and Software Program (WASP) funded by the Knut and Alice Wallenberg Foundation.
\end{ack}

\bibliographystyle{plainnat}
\bibliography{main}

\section*{Checklist}

\begin{enumerate}

\item For all authors...
\begin{enumerate}
  \item Do the main claims made in the abstract and introduction accurately reflect the paper's contributions and scope?
    \answerYes{The introduction refers to the section where each contribution is made.}
  \item Did you describe the limitations of your work?
    \answerYes{See e.g., the discussion in Section~\ref{sec:conclusion} and on the applicability of Theorem~\ref{thm:main} in Section~\ref{sec:fixed_rep} and \ref{sec:kernels_and_maps}.}
  \item Did you discuss any potential negative societal impacts of your work?
    \answerNA{We have not identified any potential negative societal impact directly related to our work.}
  \item Have you read the ethics review guidelines and ensured that your paper conforms to them?
    \answerYes{}
\end{enumerate}

\item If you are including theoretical results...
\begin{enumerate}
  \item Did you state the full set of assumptions of all theoretical results?
    \answerYes{See Assumption~\ref{asmp:linlatent} and Theorem~\ref{thm:main}.}
        \item Did you include complete proofs of all theoretical results?
    \answerYes{Yes, see the Appendix.}
\end{enumerate}

\item If you ran experiments...
\begin{enumerate}
  \item Did you include the code, data, and instructions needed to reproduce the main experimental results (either in the supplemental material or as a URL)?
    \answerYes{A URL to a code repository is included in the paper.}
  \item Did you specify all the training details (e.g., data splits, hyperparameters, how they were chosen)?
    \answerYes{See Section~\ref{sec:experiments} and Appendix~\ref{app:experiments_and_processing}}
        \item Did you report error bars (e.g., with respect to the random seed after running experiments multiple times)?
    \answerYes{Each result figure includes error regions.}
        \item Did you include the total amount of compute and the type of resources used (e.g., type of GPUs, internal cluster, or cloud provider)?
    \answerYes{}{See Section~\ref{sec:experiments}}
\end{enumerate}

\item If you are using existing assets (e.g., code, data, models) or curating/releasing new assets...
\begin{enumerate}
  \item If your work uses existing assets, did you cite the creators?
    \answerYes{See Appendix~\ref{app:adni}}
  \item Did you mention the license of the assets?
    \answerYes{See Appendix~\ref{app:adni}}
  \item Did you include any new assets either in the supplemental material or as a URL?
    \answerYes{Only original code.}
  \item Did you discuss whether and how consent was obtained from people whose data you're using/curating?
    \answerNA{}
  \item Did you discuss whether the data you are using/curating contains personally identifiable information or offensive content?
    \answerYes{ADNI data is anonymized.}
\end{enumerate}

\item If you used crowdsourcing or conducted research with human subjects...
\begin{enumerate}
  \item Did you include the full text of instructions given to participants and screenshots, if applicable?
    \answerNA{}
  \item Did you describe any potential participant risks, with links to Institutional Review Board (IRB) approvals, if applicable?
    \answerNA{}
  \item Did you include the estimated hourly wage paid to participants and the total amount spent on participant compensation?
    \answerNA{}
\end{enumerate}

\end{enumerate}


\clearpage
\appendix

\section*{Appendix}

\section{Proof of Theorem~\ref{thm:main}}
\label{app:proof}
Our proof requires an additional technical assumption: that the matrix of true latent states ${\bZ_t=[\Phi(x_{1,t}), ..., \Phi(x_{m,t})]^\top}$, for all $t$, for a random data set $D$ has independent columns with probability~1. This implies that $\mbox{rank}(\bZ_t) = d$ and $m \geq d$. We consider this a minor restriction since it would only be violated if either a) two or more components of $Z_t$ were perfectly correlated---in this case, a smaller system with same distributions over observations could always be constructed---or b) if we observe fewer samples than necesary to determine the system ($m < d$). Note that this \emph{does not} require that the dimension $\hd$ of the estimated representation $\hPhi$ is smaller than $m$.
We begin by proving that both classical and LuPTS estimators are invariant to a particular form of linear transformation of the representation $\hPhi$.

\begin{thmlem}{}\label{lem:linear_invariant}
Assume we have a latent linear Gaussian system as defined in Assumption~\ref{asmp:linlatent} such that for a data set $D$ of $m$ samples, the matrix of true latent states $\bZ_t = [\Phi(X_{1,t}); ...; \Phi(X_{m,t})]$ has linearly independent columns with probability 1. Let $\mathscr{A}_{\PI}^{{\Phi}}$ be the LuPTS algorithm using the system's true map $\Phi(\cdot)$ with $\Phi(\Psi(z))=z~\forall z \in \cZ$. Let $\mathscr{A}_{\PI}^{\hat{\Phi}}$ be the same algorithm using a different map $\hat{\Phi}(\cdot)$. We assume that $\exists
B : \hat{\Phi}(x)=B\Phi(x) \forall x\in\cX$.  Analogously, we denote the classical learners $\mathscr{A}_{\CL}^{\hat{\Phi}}$ and $\mathscr{A}_{\CL}^{\Phi}$. If $B\in\bbR^{\hat{d}\times d}$ has linearly independent columns we have
    \begin{flalign}
        \hat{h}_{\PI}^{\hat{\Phi}}(x) &=   \hat{h}_{\PI}^{\Phi}(x) \nonumber\\
        \text{and}~~\hat{h}_{\CL}^{\hat{\Phi}}(x) &=   \hat{h}_{\CL}^{\Phi}(x). \nonumber
    \end{flalign}
\end{thmlem}
\begin{proof}
Let $\bZ_t \in \bbR^{m \times d}$ be made up of the rows $\bZ_{t(i,:)} = \Phi(\bX_{t(i,:)})$ when $\bX_t\in \bbR^{m\times k}$ is the design matrix belonging to data set $D$.  In the same fashion we define $\hat{\bZ}_t\in\bbR^{m\times\hat{d}}$ using the map $\hat{\Phi}$ instead. By assumption, $\bZ_t$ and $ B\in\bbR^{\hat{d}\times d}$ have independent columns such that $ B^\dagger B=I$. These assumptions are used for matrix identities involving the Moore-Penrose inverse below. We compute the prediction on a new test point $x$ for the classical learner:
\begin{flalign*}
    \hat{h}_{\CL}^{\hat{\Phi}}(x) &= \big{(}(\hat{ \bZ}_1^\top\hat{ \bZ}_1)^\dagger \hat{ \bZ}_1^\top  Y\big{)}^\top\hat{\Phi}(x)
    \\&=\bigg{(}\big{(}( \bZ_1  B^\top)^\top( \bZ_1  B^\top)\big{)}^\dagger( \bZ_1  B^\top)^\top  \bY \bigg{)}^\top  B \Phi(x)\\
    &=\bigg{(}(\bZ_1 B^\top)^\dagger (B \bZ_1^\top)^\dagger ( \bZ_1  B^\top)^\top  \bY \bigg{)}^\top  B \Phi(x)\\
    &=\bigg{(}(B^\top)^\dagger \bZ_1^\dagger (\bZ_1^\top)^\dagger B^\dagger ( \bZ_1  B^\top)^\top  \bY \bigg{)}^\top  B \Phi(x)\\
    &= \bigg{(} ( B^\top)^\dagger \big{(} \bZ_1^\top  \bZ_1\big{)}^\dagger \underbrace{ B^\dagger  B}_{=I}  \bZ_1^\top  \bY\bigg{)}^\top  B \Phi(x)\\
    & = \bigg{(} \big{(} \bZ_1^\top  \bZ_1\big{)}^\dagger  \bZ_1^\top  \bY\bigg{)}^\top \underbrace{ B^\dagger  B}_{=I} \Phi(x)\\
    &= \hat{h}_{\CL}^{{\Phi}}(x)
\end{flalign*}
The arguments for the privileged learners are analogous:
\begin{flalign*}
     \hat{h}_{\PI}^{\hat{\Phi}}(x)
     &= \bigg{[} (\hat{\bZ}_1^\top \hat{\bZ}_1)^\dagger \hat{\bZ}_1^\top\hat{\bZ}_2 ~(\hat{\bZ}_2^\top \hat{\bZ}_2)^\dagger \hat{\bZ}_2^\top \hat{\bZ}_3~\dots~  (\hat{\bZ}_T^\top \hat{\bZ}_T)^\dagger \hat{\bZ}_T^\top \bY  \bigg{]}^\top \hat{\Phi}(x)\\
     &= \bigg{[} ( B^\top)^\dagger (\bZ_1^\top\bZ_1)^\dagger  B^\dagger  B \bZ_1^\top\bZ_2  B^\top( B^\top)^\dagger (\bZ_2^\top\bZ_2)^\dagger  B^\dagger  B \bZ_2^\top \bZ_3 B^\top
\end{flalign*}
\begin{flalign*}
     &~~~~\dots~ ( B^\top)^\dagger (\bZ_T^\top\bZ_T)^\dagger  B^\dagger  B \bZ_T^\top  \bY \bigg{]}^\top  B\Phi(x)\\
     &=\bigg{[} (\bZ_1^\top\bZ_1)^\dagger \bZ_1^\top\bZ_2  (\bZ_2^\top\bZ_2)^\dagger \bZ_2^\top \bZ_3 (\bZ_T^\top\bZ_T)^\dagger  \bZ_T^\top  \bY \bigg{]}^\top  B^\dagger B\Phi(x)\\
     &= \hat{h}_{\PI}^{\Phi}(x)
\end{flalign*}
\end{proof}

\paragraph{Proof of Theorem~\ref{thm:main}.}
We consider the generalized LuPTS estimator $\hat{h}_{\PI}^{\hPhi}(\cdot)$ treating different cases for the number of samples $m$ and latent state dimension $d$ in turn.  By the added technical assumption, that the true latent state $\bZ_t = [\Phi(x_{1,t}), ..., \Phi(x_{m,t})]^\top$ has rank $d$ for all $t$ with probability 1, and the assumption that $\hPhi$ is linearly close to $\Phi$, by a matrix $B\in\bbR^{\hat{d} \times d}$ of rank $d$ such that $\hbZ_t = \bZ_t B^\top$, we get that $\mbox{rank}(\hbZ_t) = d$ for all $t$. This also implies $d\leq\hat{d}$.

(i) $m = d$: In this case, the Gram matrix $\bK_t = \hbZ_t \hbZ_t^\top$ has full rank and thus is invertible for all $t$. By Proposition~1, LuPTS coincides with the classical learner. Hence, $\E_{\hat{h}_\PI, X_1}[\V_D(\hat{h}_\CL(X_1) \mid \hat{h}_\PI)]=0$ and $\olR(\scrA_\PI) = \olR(\scrA_\CL)$.

(ii) $m < d$:  In this case, there does not exists a linearly close, as defined above, representation $\hPhi$ to $\Phi$ since the rank of $\hbZ_t$ must be smaller than $d$. This contradicts that $\mbox{rank}(\hbZ_t) = d$. Independently, if the conditions of Proposition~\ref{prop:coincide} hold, the same equivalence holds as in the case $m=d$.

(iii) $m > d$: In this case, the kernel Gram matrix $\bK_t = \hbZ_t \hbZ_t^\top$ has rank $d < m$ and is never invertible.

Three sub-cases remain: a) When $\hd = d$, the matrix $B$ is invertible and square, the covariance matrix $\hat{\Sigma}_t = \hbZ_t^\top \hbZ_t$ is invertible for all $t$ and our estimator coincides with linear LuPTS~\citep{karlsson2021using} in the space implied by $\hPhi$. To see this, note that  Lemma~\ref{lem:linear_invariant} implies that $\hat{h}_{\PI}^{\hPhi}(\cdot)$ makes the same predictions as a different generalized LuPTS estimator $\hat{h}_{\PI}^{\Phi}(\cdot)$ using the true map $\Phi$ when the two representation functions are related through $B$, as defined in the Theorem statement. Consequently, we may analyze the latter estimator instead of the first. It uses the parameter
$$\htheta_\PI \coloneqq \Bigg[\prod\limits_{t=1}^{T-1} \underbrace{(\bZ_t^\top\bZ_t)^\dagger \bZ_t^\top \bZ_{t+1}}_{\hA_t}\Bigg] \underbrace{(\bZ_T^\top\bZ_T)^\dagger \bZ_T^\top \bY}_{\hbeta}~.$$ We know by assumption that the covariance matrices $\Sigma_t = (\bZ_t^\top\bZ_t)\in\bbR^{d\times d}$ have full rank for all $t$. This implies that the Moore-Penrose pseudoinverse $(\cdot)^\dagger$ may be replaced by the regular matrix inverse $(\cdot)^{-1}$ in the expression above, yielding
$$\htheta_\PI = \Bigg[\prod\limits_{t=1}^{T-1} \underbrace{(\bZ_t^\top\bZ_t)^{-1} \bZ_t^\top \bZ_{t+1}}_{\hA_t}\Bigg] \underbrace{(\bZ_T^\top\bZ_T)^{-1} \bZ_T^\top \bY}_{\hbeta}$$
which is equivalent to the LuPTS estimator of \citet{karlsson2021using} used on a linear-Gaussian system in the space $\cZ$ rather than in $\cX$. In this case, Theorem~1 from \citet{karlsson2021using} yields the desired result.
b) If $\hd > d$, $\hat{\Sigma}$ is not invertible but, due to Lemma~1, we can instead study a representation which is an appropriate linear transform $B \in \bbR^{\hd \times d}$ away from $\hbZ$, and apply the \citet{karlsson2021using} result as described for the case $\hd=d$. Note that in this case $B$ is non-square but has linearly independent columns as required. c) If $\hd < d$, the assumed matrix $B$ cannot exist with the stated conditions (the assumptions of Theorem~\ref{thm:main} are not satisfied).

\qed



\section{Universality of random features}
\label{app:universality}
A learning algorithm $\scrA$ is said to be universally consistent if, for any continuous function $h$, the output of $\scrA$ converges in probability to $h$. That is, for a random dataset $D_m$ of $m$ i.i.d. samples drawn from a distribution $p$, and any $\epsilon > 0$,
$$
\lim_{m\rightarrow \infty} \mbox{Pr}[\|\scrA(D_m) - h\|_{L^2(p)} > \epsilon] = 0~.
$$
\citet{sun2019approximation} prove that (norm-bounded) linear regression applied to random ReLU features (RRF) is universally consistent. Specifically, for any $\epsilon, \delta > 0$, there is a finite number of random features $\hd$ and samples $m$, such that the estimator
$$
\hh_{\mathrm{RRF}}(x) = \htheta^\top \hPhi^{\gamma, \hd}_{\mathrm{RRF}}(x) \;\;\mbox{ with }\;\; \htheta = \argmin_{\theta : \|\theta\|^2_2 < R^2} \frac{1}{m}\sum_{i=1}^m (\theta^\top \hPhi^{\gamma, \hd}_{\mathrm{RRF}}(x_i) - y_i)^2
$$
achieves an error of at most $\epsilon$ with probability $\geq 1-\delta$ for univariate continuous functions of $x$.  Universal consistency (as $\hat{d}\rightarrow \infty, m>\hat{d})$) follows as a result. We can apply the same idea to a version of generalized LuPTS by considering each parameter estimate of the latent dynamical system given $\hPhi$, with norms restricted by $R$. We drop the subscript $\mathrm{RRF}$ from $\hPhi$ moving forward, and continue to use random ReLU features. For the final prediction step of $Y$, we let
\begin{align}
\hbeta & = \argmin_{b : \|b\|^2_2 < R^2} \frac{1}{m}\sum_{i=1}^m (b^\top \hPhi^{\gamma, \hd_T}(x_{i,T}) - y_i)^2 \label{eq:normB}
\end{align}

Progressing recursively backward from $t=T$, we let
\begin{align}
[\hA_t]_{j,:} & = \argmin_{a : \|a\|^2_2 < R^2} \frac{1}{m}\sum_{i=1}^m (a^\top \hPhi^{\gamma, \hd_{t}}(x_{i,t}) - \hPhi^{\gamma, \hd_{t+1}}(x_{i,t+1})_j)^2 \;\mbox{ for }\; j\in [\hd_{t+1}], t\in [T-1] \label{eq:normAt}
\end{align}
where $\hPhi^{\gamma, \hd_{t}}$ are random features specific to $t$. Since the target of each prediction at $t+1$ is fixed with respect to the features used as input at $t$, the result from \citet{sun2019approximation} can be used to give a learning guarantee for each step. The construction differs from the standard generalized LuPTS formulation as $\hPhi$ is not shared between time steps, and so $\hA_t$ will be non-square in general. We believe that this is merely a limitation of the proof technique and that the results hold for shared random features and square transitions.

Let $\cH_{\mathrm{ReLU}}$ denote the set of linear functions applied to $\hd$ random ReLU features, with uniform random projection coefficients $\omega_j \sim \cU(\mathbb{S}^d)$, $j = 1, ..., \hd$.
$$
\cH_{\mathrm{ReLU}, \hd} = \left\{ h:\cX\xrightarrow[]{} \bbR, h(x) = \sum_{j=1}^{\hd} a_j \sigma( \omega_j^\top [x; 1])   \right\}~.
$$

\begin{thmcol}[Follows from Proposition~5 in \citet{sun2019approximation}]\label{corr:cons}
Let Assumption~\ref{asmp:linlatent} hold with noiseless transitions and outcomes, $\epsilon_t = 0$ for $t=2, ..., T$, and $\epsilon_Y = 0$.
Define $g^*_t(x_t) \coloneqq \Psi(A_t^\top \Phi(x_t)) = x_{t+1}$ and assume that for any fixed RRF representation $\hPhi^{\gamma,\hd_{t+1}}$, each component $j=1, ..., \hd_{t+1}$ of the transition target satisfies $\hPhi^{\gamma,\hd_{t+1}}(g^*_t( \cdot ))(j) \in \cH_{\mathrm{ReLU}, \hd_{t}}$. Let $\hA_t$ be the minimizer of the single-step transitions as defined in \eqref{eq:normAt}. Then, for any $\delta>0, \epsilon>0$ there is a number of random features $\hd = \hd(\epsilon, \delta)$ and samples $m = m(\epsilon, \delta)$, such that with probability $\geq 1-\delta$,
$$
\|\hA_t^\top \hPhi^{\gamma,\hd_t}(x_t) - \hPhi^{\gamma,\hd_{t+1}}(x_{t+1})\| \leq \epsilon~.
$$
\end{thmcol}
The result follows from Proposition 5 in  \citet{sun2019approximation} applied to the transition functions in our problem.
Putting this together for all time-steps, we get the following result.
\begin{thmprop}[Universal consistency of RRF privileged learner]\label{prop:consistency}
Let Assumption~\ref{asmp:linlatent} hold with $\epsilon_t = 0, \epsilon_Y = 0$. By Corollary\ref{corr:cons} and \citet{sun2019approximation}, we have for any $\delta>0$, $\epsilon>0, \gamma>0$ and a sequence $(\hd_1, ..., \hd_T)$ of sufficiently large numbers of random features and samples $m$, that with probability at least $1-\delta/T$, for the privileged estimator defined in \eqref{eq:normAt}, \eqref{eq:normB},
$$
\|\hbeta^\top \hPhi^{\gamma, \hd_T}(X_T) - Y \|_{L^2(p)} \leq \epsilon
$$
and
$$
\forall t=1, ..., T-1 : \|\hA_t^\top \hPhi^{\gamma, \hd_t}(X_t) - \hPhi^{\gamma, \hd_{t+1}}(X_{t+1}) \|_{L^2(p)} \leq \epsilon,
$$
Then, further assume that the largest eigenvalue $\lambda_{\mathrm{max}}(\hA_t) \leq C$ for any $t$ and $\|\hbeta\| \leq C$. Then, with probability at least $1-\delta$,
$$
\|(\hA_1 \cdots \hA_{T-1}\hbeta)^\top \hPhi^{\gamma, \hd_1}(X_1) - Y \|_{L^2(p)} \leq TC^T\epsilon
$$
\end{thmprop}
\begin{proof}
Let $\| \cdot \|= \| \cdot \|_{L^2(p)}$. Then, letting $\hPhi^t = \hPhi^{\gamma, \hd_t}$, and applying a union bound to each of the $T$ $(\epsilon, \delta)$-assumptions, and a series of Cauchy-Schwarz inequalities,
\begin{align*}
& \|(\hA_1 \cdots \hA_{T-1}\hbeta)^\top \hPhi^1(X_1) - Y \| \\
& = \|(\hA_1 \cdots \hA_{T-1}\hbeta)^\top \hPhi^1(X_1) - \hbeta^\top \hPhi^T(X_T) + \hbeta^\top \hPhi^T(X_T) - Y \| \\
& \leq \|(\hA_1 \cdots \hA_{T-1}\hbeta)^\top \hPhi^1(X_1) - \hbeta^\top \hPhi^T(X_T)\| + \underbrace{\|\hbeta^\top \hPhi^T(X_T) - Y \|}_{\leq \epsilon} \\
& \leq \|(\hA_1 \cdots \hA_{T-1}\hbeta)^\top \hPhi^1(X_1) - \hbeta^\top \hA_{T-1}^\top \hPhi^{T-1}(X_{T-1}) \\
& + \hbeta^\top (\hA_{T-1}^\top \hPhi^{T-1}(X_{T-1}) - \hPhi^T(X_T))\| + \epsilon \\
& \leq \|(\hA_1 \cdots \hA_{T-1}\hbeta)^\top \hPhi^1(X_1) - (\hA_{T-1} \hbeta)^\top \hPhi^{T-1}(X_{T-1})\| \\
& + \underbrace{\|\hbeta^\top (\hA_{T-1}^\top \hPhi^{T-1}(X_{T-1}) - \hPhi^T(X_T))\|}_{\leq C\epsilon} + \epsilon \\
& ... \\
& \leq TC^T \epsilon~.
\end{align*}
\end{proof}
\begin{thmrem}
Proposition~\ref{prop:consistency} shows that a privileged learner with random ReLU features can be turned into a universally consistent estimator of any (noiseless) continuous function of $X_1$ by letting each time step have its own random feature representation of appropriate size and adding a norm constraint to each linear transformation.
The construction in \eqref{eq:normAt}, \eqref{eq:normB} deviates from Algorithm~\ref{alg:genlupts} primarily in that the random feature representations used at each time step are different, but the overall structure is maintained: At training, predictions of $Y$ are  made from an embedding $\hZ_T=\hPhi^T(X_T)$ of $X_T$, predictions of $\hZ_T$ are made from an embedding $\hZ_{T-1}=\hPhi^{T-1}(X_{T-1})$, and so on. At test time, the transition matrices $\hA_1, ..., \hA_T, \hbeta$ are multiplied and applied to $\hZ_1 =\hPhi^1(X_1)$. We conjecture that a similar argument can be applied to the construction in Algorithm~\ref{alg:genlupts} by letting the dimension of each time step approach $\infty$.
\end{thmrem}

\section{Compounding bias}
\label{app:compounding_bias}
We can describe the compounding bias of the LuPTS estimator due to a biased representation $\hPhi$, in comparison with the standard OLS estimator, by propagating the error in $\hPhi$ through the estimates. Assume that
$$
Y = \theta^\top \Phi(X_1) + \epsilon = (A_1 \cdots A_{T-1}\beta)^\top \Phi(X_1) + \epsilon'
$$
Then, let $Z_t = \Phi(X_t)$ and for an estimate $\hPhi$, assumed for simplicity to have the same dimension, $\hd = d$,
$$
\hZ_t = \hPhi(X_t) = \Phi(X_t) + R_t
$$
where $R_t$ is the residual w.r.t. $\Phi$. Let bold-face variables indicate multi-sample equivalents of all variables. Further, define $\Sigma_t = \bZ_t^\top \bZ_t$ and $\hSigma_t = \hbZ_t^\top \hbZ_t$.

Fitting $\htheta$ to $\hPhi$ using the classical learner (OLS) yields an estimate
$$
\htheta_c = \hSigma_1^{-1}\hbZ_1^\top \bY
$$
Now, define $\Omega_t = \bR_t^\top\hbZ_t + \bZ_t^\top \bR_t$ and we have
\begin{align*}
\htheta_c & = (\Sigma_1 + \Omega_1)^{-1}(\bZ_1 + \bR_1)^\top \bY \\
& = (\Sigma_1^{-1} + \Delta_1)(\bZ_1 + \bR_1)^\top \bY \\
& = \htheta_c^* + (\Delta_1\hbZ_1^\top + \Sigma_1^{-1}\bR_1^\top) \bY
\end{align*}
where $\Delta_1 = -\Sigma_1^{-1}\Omega_1(\Sigma_1  + \Omega_1)^{-1}$, $\htheta_c^*$ is the OLS estimate of $\theta$ for the true $\Phi$ and the second line follows from the Woodbury matrix identity. The norm of $\Delta_1$ is related to the condition number of $\Sigma_1$.  The expectation of the first term is $\theta$, and the expectation of the remaining terms is the bias.

Now, we can do the same thing for the privileged estimator. Let's start with $T=2$.
\begin{align*}
\htheta_p & = \hSigma_1^{-1}\hbZ_1^\top \hbZ_2 \hSigma_2^{-1}\hbZ_2^\top Y \\
& = (\Sigma_1^{-1} + \Delta_1)(\bZ_1 + \bR_1)^\top (\bZ_2 + \bR_2) (\Sigma_2^{-1} + \Delta_2)(\bZ_2 + \bR_2)^\top \bY \\
& = \htheta_p^* + \hA_1(\Sigma_2^{-1}\bR_2^\top + \Delta_2\hbZ_2^\top)\bY + (\Sigma_1^{-1} \hbZ_1^\top \bR_2 + \Sigma_1^{-1} \bR_1^\top \hbZ_2 + \Delta_1\hbZ_1^\top \hbZ_2)\hbeta~.
\end{align*}
Thus, the difference in bias between the two estimators is
\begin{align*}
\E[\htheta_c - \htheta_p]  = &\underbrace{\E[\htheta_c^* - \htheta_p^*]}_{= 0} \\
& + \E[(\Delta_1\hbZ_1^\top + \hSigma_1^{-1}\bR_1^\top) \bY - \hA_1(\Delta_2\hbZ_2^\top + \hSigma_2^{-1}\bR_2^\top )\bY \\
& - (\Sigma_1^{-1} \hbZ_1^\top \bR_2 + \Sigma_1^{-1} \bR_1^\top \hbZ_2 + \Delta_1\hbZ_1^\top \hbZ_2)\hbeta]
\end{align*}
More generally, we can express this difference recursively as below.


\begin{thmprop}
Let $\htheta_p \coloneqq \hA_1 \cdots \hA_T \hbeta$ be a privileged estimator using a linearly biased representation $\hPhi$, and let $\hA_t^*$ be the same estimator using an unbiased representation $\Phi^*$. Then, the bias of $\htheta_p$ is
$$
\E[\htheta_p - \theta] = \E[\htheta_p - \htheta_p^*] =  \E[E_T \hbeta^* + (\hA_1 \cdots \hA_T)(\hbeta - \hbeta^*)]~,
$$
where $E_t$ is the compounded error in transition dynamics, computed recursively as follows
\begin{align*}
E_t \coloneqq (\hA_1 \cdots \hA_t) - (\hA^*_1 \cdots \hA^*_t)  = E_{t-1} \hA_{t}^* + (\hA_1 \cdots \hA_{t-1}) (\hA_{t} - \hA^*_{t})~,
\end{align*}
with $E_0 = 0$. In the worst case, the bias of $\htheta_p$ grows exponentially with $T$.
\end{thmprop}

\section{Privileged time series representation learners}
\label{app:GRL}

We expand on the description of the greedy representation learner (GRL) described as a special case of CRL in Section~\ref{sec:learn_rep}. To avoid the information loss of SRL, we consider its conceptual opposite, using privileged time series information \emph{only} to predict the outcome. To do this, a linear output layer $\htheta_t^\top \hZ_t$ is used to predict $Y$ at every time step $t$. Recall that, by Assumption~\ref{asmp:linlatent}, the expected outcome is linear in the latent state at \emph{any} time step. The method is related to multi-view learning, in which prediction of the same quantity is made from multiple ``views''~\citep{multi-view_survey}. We dub the model \emph{greedy privileged representation learner} (GRL), which minimizes the objective
\begin{equation}
    \mathcal{L}_\GRL(\hPhi, \{\htheta_t\}) := \frac{1}{NT} \sum\limits_{i=1}^N \sum\limits_{t=1}^{T} w_t\big{\Vert} \htheta_t^\top\hPhi(x_{i,t}) - y_i\big{\Vert}_2^2.
    \label{eq:GRL}
\end{equation}
During inference this algorithm returns $\hat{h}_\PI(\cdot)=\htheta_1^\top\hPhi(\cdot)$. Compared to the objective of CRL in \ref{eq:CRL}, we introduce an additional hyperparameter $\lambda\in (0,1)$ to place more weight on the loss term that is relevant at inference time. As a consequence, we choose $w_t =\lambda$ for $t=1$ and $w_t=1-\lambda$ otherwise. We expect GRL to have less bias than SRL, but higher variance since less structure is imposed on the representation $\hPhi$.

\section{Experiment setup \& data processing}
\label{app:experiments_and_processing}
\subsection{Detailed experiment setup}

In the following we give a detailed description of the experimental setup used to obtain the results presented in Section~\ref{sec:experiments} as well as the additional results that are part of this section. For a given data set, we select a combination of training set sizes and sequence length. For each unique combination of these parameters the models of interest are trained repeatedly with different random sampling. For each repetition the data is split into a train and a test set randomly before hyperparameter tuning and model training are performed. At last each model's predictions on the test set are scored by computing the coefficient of determination $R^2$. On synthetic data the test set contains 1000 samples. In the case of real-world data, where samples are limited, we test on 20\% of all available data.

The preprocessing used for real-world data and the generation procedure of synthetic data is unique to each data set. We refer to the data set specific subsections for detailed descriptions of how each data set is processed. During the experiments each model uses standardized data for training and inference. To perform the data rescaling we use the StandardScaler implementation that is part of scikit-learn \citep{sklearn}.

\paragraph{Hyperparameter tuning.} The tuning of hyperparameters is carried out for each repetition and is implemented using random search and five-fold cross-validation. Each hyperparameter is sampled from a fixed interval of possible values. An overview of the ranges of different hyperparameters determined through random-search is provided in Table~\ref{tab:hyperparameters}. For the experiments with variants of generalized LuPTS we sample ten sets of hyperparameters before retraining on all training data using the best set of parameters. For the representation learners we merely sample five values for $\lambda$.

\begin{table}[t!]
    \centering
    \begin{tabular}{lllc}
          \toprule
          Hyperparameter & Description &  Used in Algorithm & Value Range \\
          \midrule
          $n_{RF}$ & \small{number of random features} & \small{all random feature methods} & $[0.05m, 0.8m]$\\
          $\gamma_{RRF}$ & \small{bandwith parameter} & \small{Random ReLU methods} & $[0.01,10]$\\
          $\gamma_{RFF}$ & \small{bandwith parameter} & \small{Random Fourier methods} & $[0.001,0.1]$\\
          $\lambda$ & \small{loss function parameter} & \small{GRL \& CRL} & $[0,1]$\\
          \bottomrule
    \end{tabular}
    \vspace{1em}
    \caption{Overview of all hyperparameter determined by hyperparameter tuning. $m$ denotes the number of samples, meaning that $n_{RF}$ is chosen from different ranges depending on the size of the training set.}
    \label{tab:hyperparameters}
\end{table}

\paragraph{Neural network training.} The training of neural networks involves many choices and hyperparameters. We choose PyTorch's implementation \citep{pytorch} of the Adam optimizer \citep{kingma2017adam} to train the representation learning models. If not specified otherwise the results shown in this project are obtained using a learning rate of $0.0001$, a batch size of $30$, leaky ReLU activations and a maximum of $1500$ training epochs. In the case of neural network models, the sample sizes reported as part of the experiments denote the combined size of the training and validation set, where the validation set contains 20\% of those samples. We use early stopping during the training process by keeping track of the validation loss. If a model does not improve the validation loss over a waiting period of $200$ epochs we stop training early and set the network parameters to the values that obtained the lowest validation loss up until that point. In order to make sure that results are not dependent on the specific choice of the parameters just described, we performed additional experiments with different parameter choices. The results were found to be robust to small changes in these parameters.

\paragraph{Generalized distillation.} In order to compare our algorithms to the alternative of using generalized distillation for privileged time series as presented by \cite{hayashi2019long}, we implemented a model that (i) produces hypotheses of the same class as our other algorithms and (ii) that adopts the learning paradigm of a student model incorporating soft targets produced by a teacher model into its loss function. For tabular data our teacher model is a multi-layer-perceptron (MLP) with $T*k$ input neurons such that all $\{X_t\}$ are concatenated and then used as input. The teacher MLP makes use of five hidden layers, each consisting of 100 neurons. In the case of the image data generated by \clocks{}, the teacher uses an implementation of LeNet-5 on all variables $X_t$ (while sharing the encoder parameters) before concatenating the 25-dimensional output of this encoder from different time steps. This combined representation is then processed by an MLP with a single hidden layer with 25 neurons. The loss function used for the student model producing the estimate $\hat{h}_{\CL}$ is architecturally identical to the classic representation learner used in each of the experiments. When training the student model the mean squared error on the data and the error corresponding to the soft targets of the teacher model are combined via a hyperparameter $\lambda$:

\begin{equation*}
    \cL_{GD} \coloneqq \lambda \frac{1}{m}\sum\limits_{i=1}^m \Vert \hat{h}_{\CL}(x_i) - y_i \Vert_2^2  + (1-\lambda) \frac{1}{m}\sum\limits_{i=1}^m \Vert \hat{h}_{\CL}(x_i) - \hat{h}_{\text{teacher}}(x_i) \Vert_2^2
\end{equation*}

The hyperparameter is determined via hyperparameter tuning in all repetitions as described for other hyperparameters in this section. All other training procedures follow the same logic as described above.

\paragraph{Resources. } For the training of the representation learning algorithms we use a cluster of graphics processing units (GPUs) in order to reach the number of experiment repetitions required for our work. A single experiment like shown in Figure~\ref{fig:RL_CLOCK_R2} takes several hours on 100 NVIDIA~Tesla~T4 GPUs. While the random feature methods do not require GPU training, they still require hyperparameter tuning which is why we compute results such as presented in Figure~\ref{fig:RF_R2} on many CPU cores in parallel. While the experiments on neural networks cannot reasonably be reproduced on a single desktop machine, this is still possible within a few days for the random feature methods.

\paragraph{Latent variable recovery and SVCCA.}
In order to assess to what extent a representation has been found that is linearly related to the true latent variables we use SVCCA as described by \citet{SVCCA}, meaning we first use PCA retaining at least 99\% of variation before then applying CCA. For the visualization shown in Figure~\ref{fig:svcca_rainbow_clock_y1} we construct a grid of points (150 $\times$ 150) around the origin and assign each point a unique color. Then we compute an image using the observation generating function $\Psi:\cZ\xrightarrow[]{}\cX$ of the \clocks{} data set for each point. The observations are passed to the encoder $\hPhi$ of the two representation learners producing estimates of the latent variables. We then map the estimates to the ground truth linearly using SVCCA before plotting the result.

\subsection{Alzheimer progression}
\label{app:adni}
To test our algorithms on the task of predicting the progression of Alzheimer's disease (AD) we use an anonymized data set obtained through the Alzheimer's Disease Neuroimaging Initiative (\adni{})~\citep{adni} under the LONI Research License. The initiative is large multi-site research study on the brains of over 2000 AD patients which collects many features such as genetic, imaging and biospecimen biomarkers.  The data consists of measurements taken every 3 months with some observations missing. The outcome of interest in our experiments is the Mini Mental State Experiment (MMSE) score 48 months after the first measurement\citep{MMSE}. Privileged information are the measurements taken between those time points, at 12, 24 and 36 months into the program.

\paragraph{Data processing.} The processing procedure used in this project is borrowed directly from the work of \cite{karlsson2021using}. There is a large amount of missing information in the ADNI data set. The missingness varies with the time of when measurements were taken. Further some subjects were not present at some of the follow-up examinations. To deal with the missingness patients without an observation for the final follow up (the outcome $Y$) are excluded from our experiments. Further, we also require that patients are present at all intermediate time steps (12, 24 and 36 months after the first measurement) which we use as privileged information. We one-hot encode categorical features and exclude features for which more than 70\% of the observations are missing. To deal with the remaining missing values, mean imputation is used. Due to the filtering that we apply as a result of the missing data we only obtain 502 suitable sequences that we can use for our experiments.

\subsection{Traffic data}

The \traffic{} data set \citep{MNTraffic} obtained through the UCI machine learning repository \citep{Dua:2019} contains hourly measurements of the traffic volume as well as weather features and holiday information. The raw data contains 48.204 records. An overview of all available features is given in Table~\ref{tab:traffic features}.

\begin{table}[!ht]
    \centering
    \begin{tabular}{lll}
    \toprule
         Feature & Type & Description  \\ \midrule
         Date Time & Timestamp &  date and time (CST)\\
         Holiday & String & name of holiday if applicable \\
         Weather Description & String & brief free text description of the weather \\
         Weather Main & Categorial & contains categories like clear, clouds, or rain\\
         Rain\_1h & Numerical & rain in $\frac{L}{h~m^2}$ \\
         Snow\_1h & Numerical & snow in $\frac{L}{h~m^2}$ \\
         Temp & Numerical & temperature in Kelvin \\
         Traffic Volume & Numerical & hourly reported westbound traffic volume \\ \bottomrule
    \end{tabular}
    \caption{Features available in the \traffic{} data set.}
    \label{tab:traffic features}
\end{table}

\paragraph{Data processing.} We noticed extreme outliers in the data set as well as implausible numerical values for the temperature and rain features. Further, records for some of the hours of the timeframe (2012 - 2018) covered by this data set are missing. To deal with the extreme outliers we calculate the mean and standard deviation of each feature and remove records which contain values that are further than six standard deviations from the mean of a particular feature. We also remove a feature entirely if there is no variation left after this filtering. This is the case for the snowfall feature as snow is very rare in Minneapolis. From the date and time of each record we calculate the weekday which we add as a one-hot encoded feature and also represent the hour of the day $h\in\{0, 1, \dots 23\}$ as two separate periodic features given by
\begin{equation}
    t_{\text{periodic}} =  \bigg{[} \sin{(\frac{2\pi \cdot h}{24})},~   \cos{(\frac{2\pi\cdot h}{24})}\bigg{]}~.
    \label{eq:periodic_time_feature}
\end{equation}
This ensures that a timestamp just before midnight produces similar features compared to just after midnight. We one-hot encode the holiday information, making no difference between different types of holidays, and make this feature persist over a full calendar day. In the original data set the holiday information is only specified for the first hour of the day. The column Weather\_Main contains some weather conditions that are very rare, such as smoke and squall. As a consequence we group the different conditions before encoding them as binary variables. In particular we make drizzle, rain and squall one single feature while also grouping together fog, haze, mist and smoke as they all affect visibility.

\paragraph{Time series selection.}After this preprocessing, that leaves only numerical values and one-hot encoded categorical values, we group the data together as time series used for the experiments. In order to do so we specify a desired sequence length $T+1$ and a sequence step size in hours. With this information we iterate through the data set assembling time series with i) no values missing ii) the correct length and step size and iii) at least a seven hour gap between each pair of sequences. The third condition is introduced to make sure one does not end up with very similar cases (for short sequences in particular) in training and test set.

\subsection{Square-Sign}
\label{sec:square-sign-appendix}
The \squaresign{} data set serves as a test environment for learning from privileged time series information where one can assure the conditions necessary for Assumption~\ref{asmp:linlatent} to hold. In particular this means creating a linear-Gaussian system which remains unobserved and combining it with an observation generating function $\psi:\cZ\xrightarrow[]{}\cX$.

\paragraph{Latent linear-Gaussian system.} The first component that makes up the generation process for \squaresign{} (and \clocks{}) is the linear-Gaussian system which is latent, just as depicted on the right side of Figure~\ref{fig:latent_dgp} with $Z_t\in\bbR^d$.
The first step in the data creation process is sampling  each of the $d$ components of $Z_0$ from $\cN(0,5)$. Then the subsequent latent variables $Z_{t+1}$ are computed as
\begin{equation*}
    Z_{t+1} \coloneqq A_{t}^\top Z_{t} + \epsilon,~\epsilon\in\bbR^d,~\epsilon_j~\cN(0,1)~.
\end{equation*}
For the outcome we use the same form but with different dimensionality:
\begin{equation*}
    Y \coloneqq \beta^\top Z_T + \epsilon_y,~\epsilon_y\in\bbR^{q}
\end{equation*}
Off-diagonal elements of the transition matrices $A_t \in \bbR^{d\times d}$ are sampled from a Normal distribution $\cN(0, 0.2)$ while the diagonal elements are set to one. In a second step we compute the spectral radius of the randomly created matrices $A_t$ via eigenvalue decomposition, obtaining the components $U\Lambda U^\top$. We then set the spectral radius to a predefined value $\lambda_{max} = 1.3$ and reassemble the matrix as
\begin{equation*}
    A_t \xleftarrow[]{} U \frac{\lambda_{max}}{\lambda_s}\Lambda U^\top~.
\end{equation*}
The coefficients of $\beta$ are drawn from the same normal distribution as the ones of $A_t$ but undergo no further changes.

\paragraph{Observation generating function.} As the dimensionality $d$ of the latent space $\cZ$ is not fixed we use an observation generating function that is not restricted to a specific value of $d$. For each element in $Z_t\in\cZ=\bbR^d$ we create two elements in $X_t\in\cX=\bbR^{2d}$ by denoting its sign separately from its square. This gives the following nonlinear observation generating function:

\begin{equation}
    X_t \coloneqq \psi(Z_{(t)})= [Z_{(t,1)}^2, \mathrm{sgn}(Z_{(t,1)}), \dots , Z_{(t,d)}^2, \mathrm{sgn}(Z_{(t,d)})]^\top
    \nonumber
\end{equation}

\subsection{Clocks-LGS}

This data set serves the purpose of testing our algorithms on a different modality with high dimensional data. In particular the idea was to use image data as this is a domain where neural networks have been very successful. For this reason we combine a latent dynamical system with an image generation process which we explain in detail in this section.

\paragraph{Latent linear-Gaussian system.} We use exactly the same setup as we do for the \squaresign{} latent dynamical system as described in Section~\ref{sec:square-sign-appendix}. The only difference here is the dimensionality of the latent variables, transition matrices and the outcome. For \clocks{} we generally have $d=2$ and $q\in\{1,2\}$.

\paragraph{Image generation.} The second part of \clocks{} is creating images from two dimensional latent vectors $Z_t=[Z_t^{(1)}, Z_t^{(2)}]^\top$. The goal was to keep it the process simple while using small black and white images of 28$\times$28 pixels. In addition we wanted each image to have no ambiguity with respect to the latent state it represents. We  represent the first component by a clock hand mounted at the image center. One can think of $Z_t^{(1)}$ as the angle in radian, meaning the hand points straight up for $Z_t^{(1)}=0$ or $Z_t^{(1)}=2\pi$ and straight down for $Z_t^{(1)}=\pi$. To visualize a full rotation we increase the size of the cirle around the image center in discrete steps for each mutliple of $2\pi$. For negative values the circle is empty (black) while it is filled (white) for positive values. For the second component we make use of the same logic but instead of a clock hand, we only use a circle that orbits the image center. The two hands cannot obscure each other as the orbiting cirle uses a larger radius. Figure~\ref{fig:clock-number-pairs} shows three examples of pairs of corresponding latent vectors and generated images.

\begin{figure}[t!]
    \centering
    \includegraphics[width=0.5\textwidth]{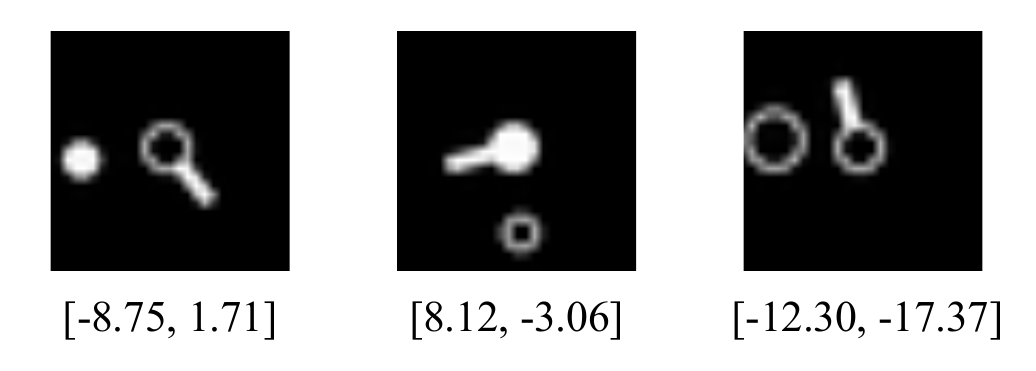}
    \caption{Pairs of corresponding latent vectors $Z_t$ and images as generated in \clocks{}.}
    \label{fig:clock-number-pairs}
\end{figure}

\subsection{$\textbf{PM}_{2.5}$ air quality}

Due to health concerns the air quality in Chinese cities has become an important topic. The \pmair{} data set contains hourly meteorologic information and the concentration of small particles ($\text{PM}_{2.5}$) for the cities Beijing, Shanghai, Guangzhou, Chengdu and Shenyang \citep{pm2.5}. The individual features available for all cities are listed in Table~\ref{tab:pm2.5_features}. In addition to the features listed, the data includes the date and time of each record. Just like in the preprocessing of \traffic{} we compute a periodic time feature using expression  \ref{eq:periodic_time_feature} to represent the time of day of each record. For each numerical feature we calculate the mean and standard deviation and remove rows with values that are more extreme than six standard deviations from the mean. We also remove rows with missing categorical features, which are then represented as one-hot vectors. Apart from differences in the preprocessing we consider the same prediction task as \cite{karlsson2021using} which is predicting the future particle concentration for a fixed time horizon given current observations.

\begin{table}[t!]
    \centering
    \begin{tabular}{lll}
    \toprule
    Feature & Type & Description\\
    \midrule
    season & Numerical & season (1 to 4) of the data in this row\\
    PM & Numerical & $\text{PM}_{2.5}$ particle concentration in $\mu g  / m^3$\\
    DEWP & Numerical & dew point in $^{\circ}$C     \\
    TEMP & Numerical & air temperature in $^{\circ}$C     \\
    HUMI & Numerical & humidity in \% \\
    PRES & Numerical & atmospheric pressure in hPa \\
    cbwd & Categorical & combined wind direction in \{N,W,S,E, NW, SW, NE, SE\} \\
    Iws & Numerical & cumulated wind speed in $m/s$ \\
    precipitation & Numerical & hourly precipitation in $mm$ \\
    Iprec & Numerical & cumulated precipitation in $mm$ \\ \bottomrule
    \end{tabular}
    \vspace{1em}
    \caption{Features available as part of the \pmair{} data set on the air quality of five large Chinese cities.}
    \label{tab:pm2.5_features}
\end{table}

\section{Additional experiment results}
\label{app:additional_experiments}
In the course of this section we present a larger scope of our experimental results.
We demonstrate the predictive accuracy of the algorithms introduced in Sections~\ref{sec:kernels_and_maps} and \ref{sec:learn_rep} in terms of the mean coefficient of determination $R^2$ over different settings on five data sets.
The variation of the results over repetitions is represented by the shaded areas in the visualizations, which denotes one standard deviation above and below the mean value. We also show empirically that the bias of generalized LuPTS increases with the number of privileged time steps when a poor representation function is used. This can be seen in Figure~\ref{fig:r2_sequence_length_square_sign} where we test privileged and classical learners over 500 different systems of the \squaresign{} type over different sequence lengths. In addition we further illustrate how privileged information can improve the recovery of latent variables of the data generating process by providing more visualizations in the style of Figure~\ref{fig:clock_svcca_scatter_y1} on the \clocks{} and \squaresign{} data set. These can be found in Figures~\ref{fig:clocks_svcca_app_y1y2} and \ref{fig:ss_svcca_scatter_app}.

In the following we first provide the experiment details for visualizing the phase transition of generalized LuPTS as seen in Figure~\ref{fig:limitations}. In the subsections thereafter the material is organized by data set.

\subsection{Two regimes of generalized LuPTS}

As seen in Figure~\ref{fig:limitations} and demonstrated by Proposition~\ref{prop:coincide}, generalized LuPTS becomes equivalent to the corresponding classical learner when the number of features $\hat{d}$ is larger than the number of samples $m$. In the following we provide the experiment details that led to Figure~\ref{fig:limitations}.

In order to evaluate the dependency on the number of features we used linear LuPTS and OLS on a synthetically generated linear-Gaussian system as displayed in Figure~\ref{fig:obs_dgp}. We use the same setup as for the latent dynamics in \squaresign{} but without a nonlinear observation generating function. For each number of features $k=\hat{d}=d$ we sample 50 such systems with different dynamics, each producing a training set with $m=100$ and test set of 1000 samples. The systems are all configured with $T=3$ and $q=10$. We train and score both estimators on all of the data generating systems before computing the mean coefficient of determination over all systems with the same number of features.

\subsection{Alzheimer progression}

\begin{figure}[H]
\centering
\begin{subfigure}[t]{\textwidth}
\centering
    \includegraphics[width=\textwidth]{fig/legends/rf.pdf}
\end{subfigure}
\begin{subfigure}[t]{\textwidth}
\centering
    \includegraphics[width=\textwidth]{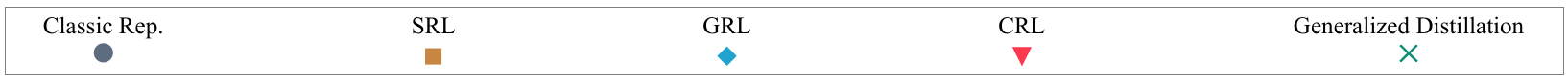}
\end{subfigure}

\begin{subfigure}[t]{0.3\textwidth}
        \includegraphics[width=\textwidth]{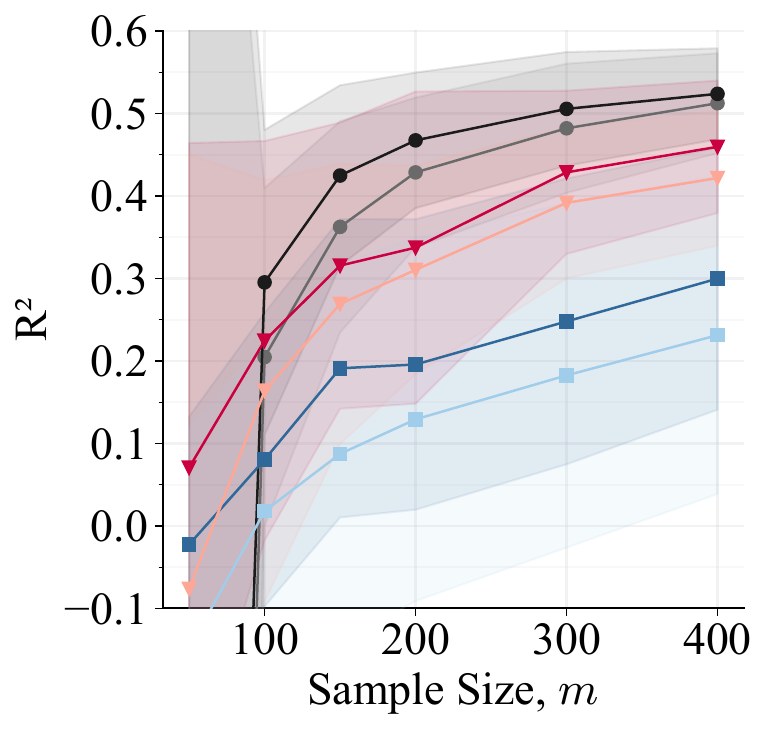}
            \caption{Generalized LuPTS on the reduced setting with a single privileged time point, $T=2$.}
\end{subfigure}
\hfill%
\begin{subfigure}[t]{0.3\textwidth}
        \includegraphics[width=\textwidth]{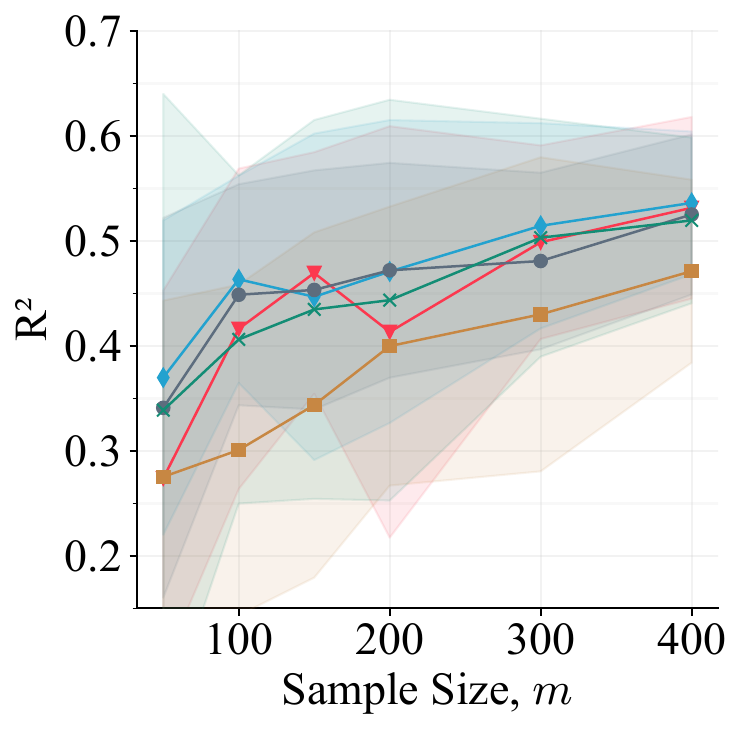}
            \caption{Representation learners and generalized distillation on the reduced setting with one privileged time step, $T=2$.}
\end{subfigure}
\hfill
\begin{subfigure}[t]{0.3\textwidth}
    \includegraphics[width=\textwidth]{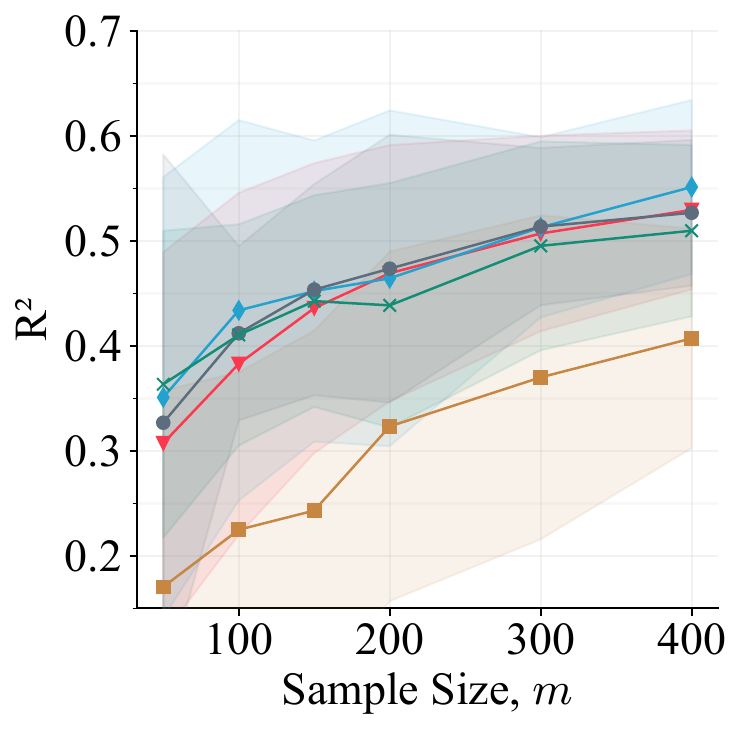}
    \caption{Representation learners and generalized distillation on the full setting with three privileged time steps, $T=4$.}
\end{subfigure}

\caption{Generalized LuPTS and the representation learning algorithms tested in terms of their predictive accuracy on different settings of the \adni{} prediction task. Each experiments are based on 25 repetitions while the random feature experiment consists of 60 repetitions.}
\label{fig:RF_R2_App}
\end{figure}

\subsection{Traffic data}

\begin{figure}[H]
\begin{subfigure}[t]{\textwidth}
\centering
    \includegraphics[width=\textwidth]{fig/legends/rf.pdf}
\end{subfigure}

\begin{subfigure}[t]{0.3\textwidth}
        \includegraphics[width=\textwidth]{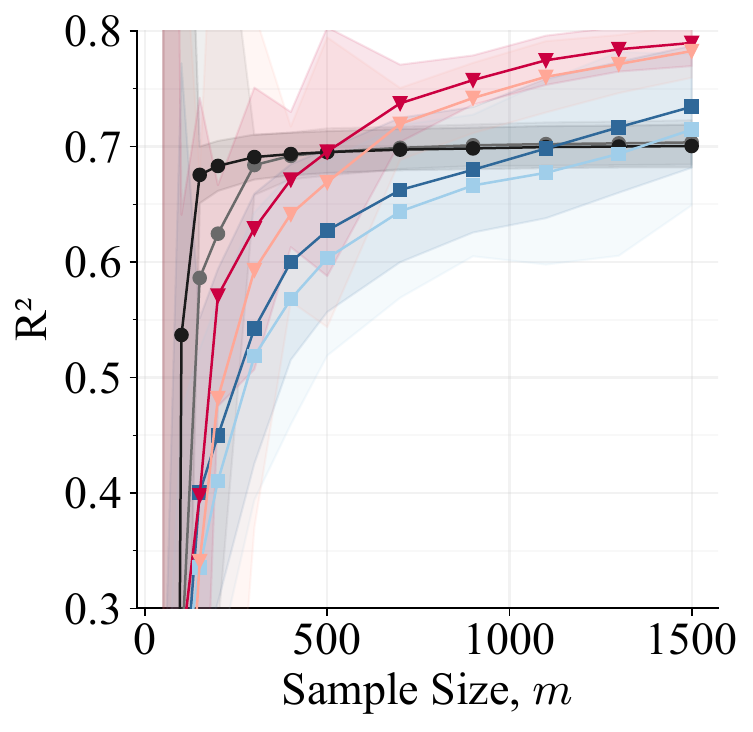}
\caption{$T=2$. }
\end{subfigure}
\hfill%
\begin{subfigure}[t]{0.3\textwidth}
\includegraphics[width=\textwidth]{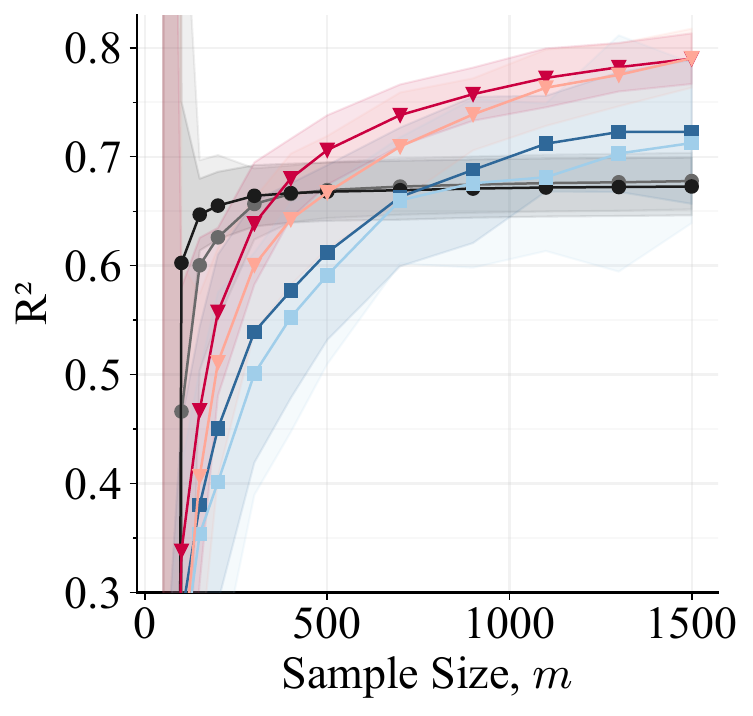}
\caption{$T=3$. }
\end{subfigure}
\hfill
\begin{subfigure}[t]{0.3\textwidth}
\includegraphics[width=\textwidth]{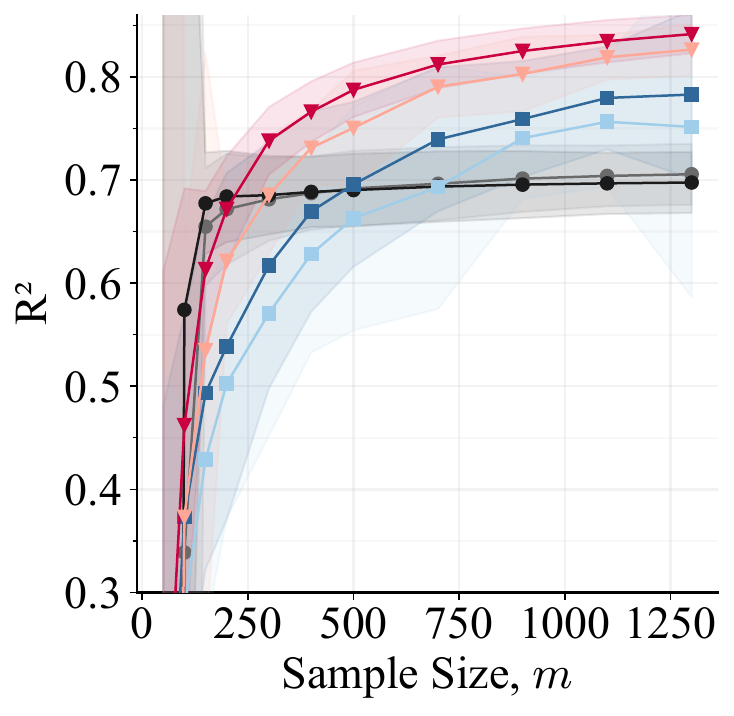}
\caption{$T=4$. }
\end{subfigure}
    \caption{Prediction accuracy of the different variants of generalized LuPTS on \traffic{} with varying sequence lengths. Based on 60 repetitions and four hour steps in each time series.}
\end{figure}

\begin{figure}[H]
\begin{subfigure}[t]{\textwidth}
\centering
    \includegraphics[width=\textwidth]{fig/legends/replearn_gd.pdf}
\end{subfigure}

\begin{subfigure}[t]{0.3\textwidth}
        \includegraphics[width=\textwidth]{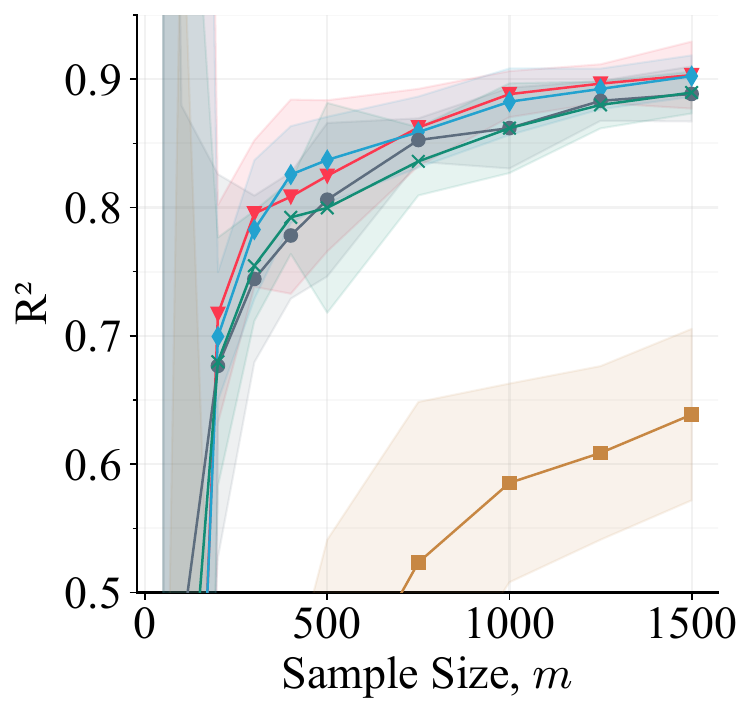}
\caption{$T=2$. }
\end{subfigure}
\hfill%
\begin{subfigure}[t]{0.3\textwidth}
\includegraphics[width=\textwidth]{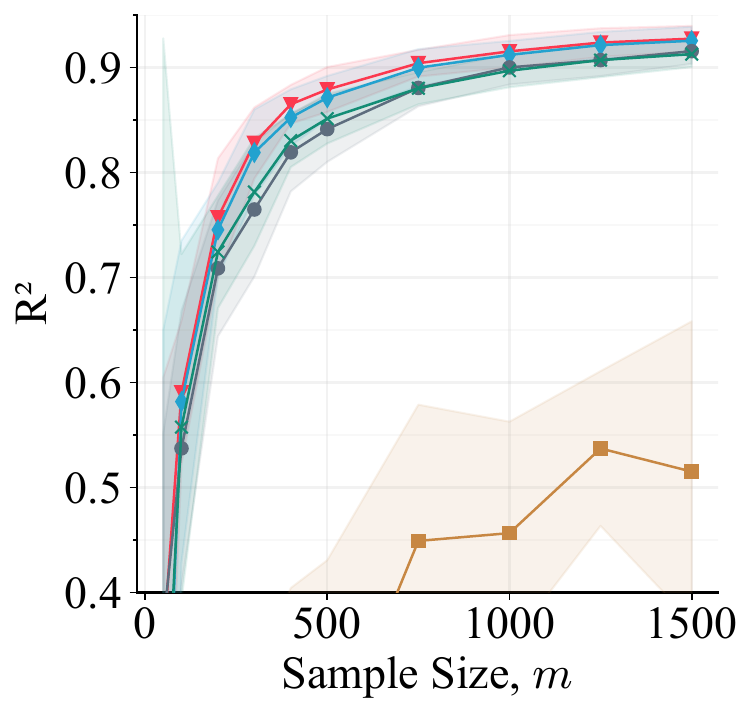}
\caption{$T=3$. }
\end{subfigure}
\hfill
\begin{subfigure}[t]{0.3\textwidth}
\includegraphics[width=\textwidth]{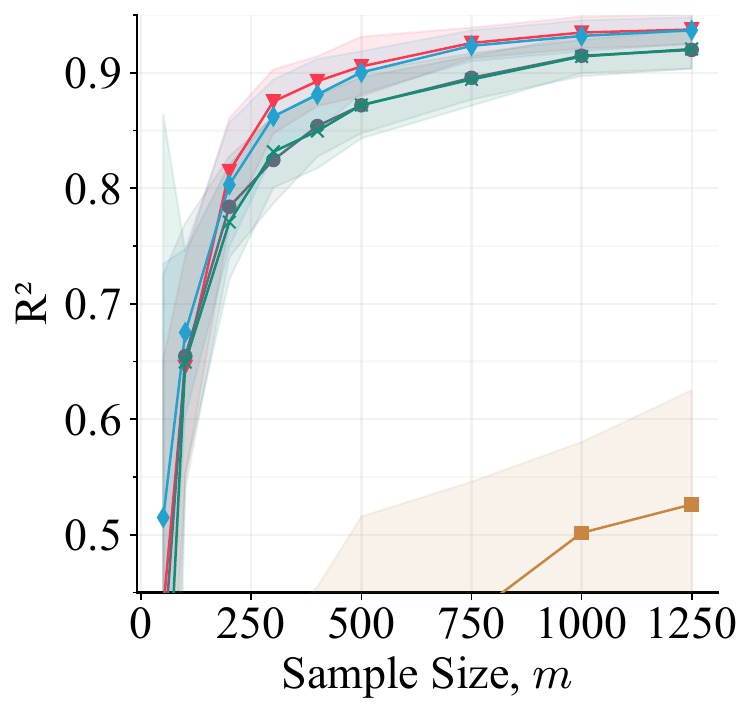}
\caption{$T=4$. }
\end{subfigure}\vspace{-0.2em}
    \caption{Prediction accuracy of the different representation learning algorithms and generalized distillation on \traffic{} with varying sequence lengths. The results are based on 25 repetitions and four hour steps in each time series.}
\end{figure}\vspace{-0.4em}

\subsection{Clock\_LGS}

\begin{figure}[H]
\begin{subfigure}[t]{\textwidth}
\centering
    \includegraphics[width=\textwidth]{fig/legends/replearn_gd.pdf}
\end{subfigure}

\begin{subfigure}[t]{0.3\textwidth}
        \includegraphics[width=\textwidth]{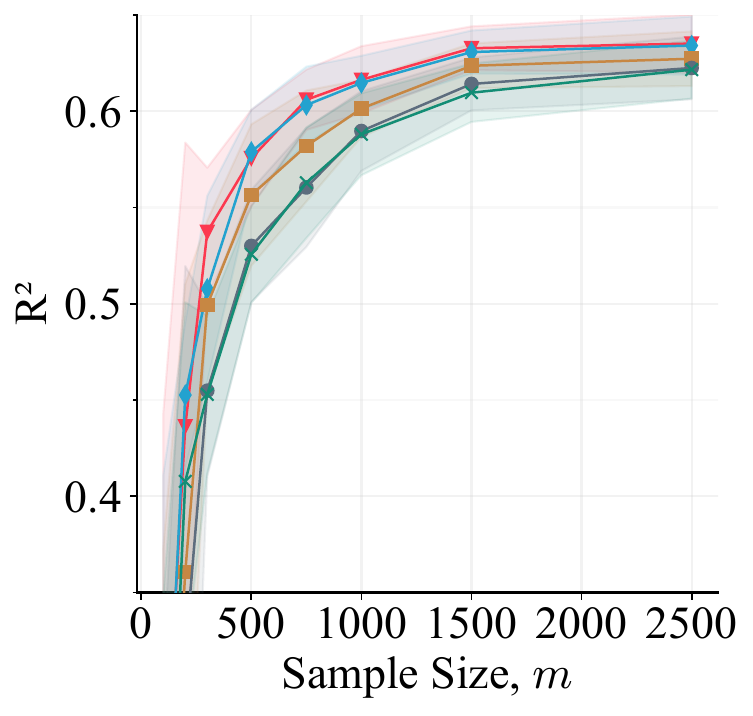}
            \caption{$T=2$.}
\end{subfigure}
\hfill%
\begin{subfigure}[t]{0.3\textwidth}
\includegraphics[width=\textwidth]{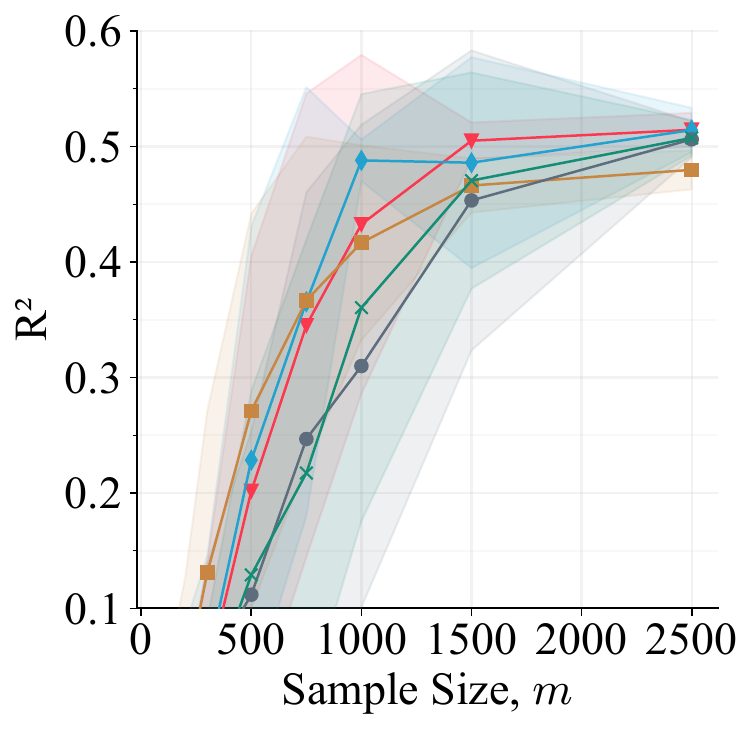}
    \caption{$T=3$.}
\end{subfigure}
\hfill
\begin{subfigure}[t]{0.3\textwidth}
\includegraphics[width=\textwidth]{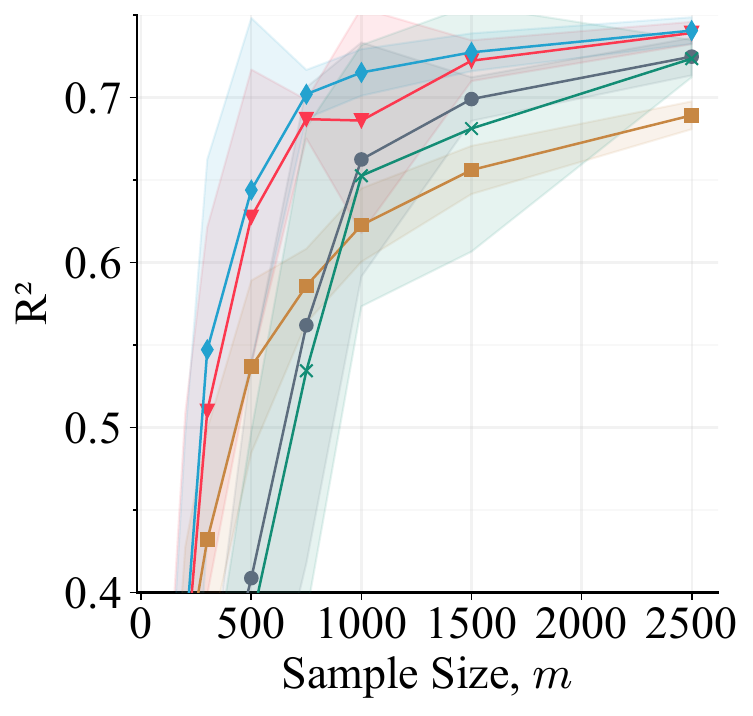}
    \caption{$T=5$.}
\end{subfigure}\vspace{-0.2em}
    \caption{Prediction accuracy of the representation learning algorithms and generalized distillation on \clocks{} with a two dimensional outcome $q=2$ and varying sequence lengths, based on 25 repetitions.}
\end{figure}\vspace{-0.4em}

\begin{figure}[H]
    \centering
\includegraphics[width=\textwidth]{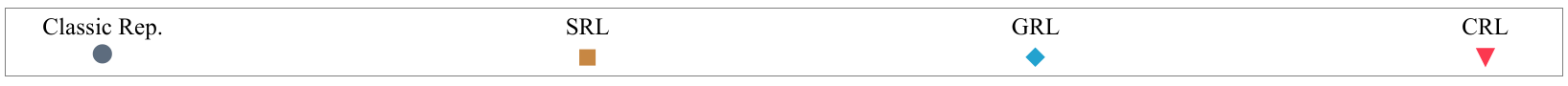}
\includegraphics[width=\textwidth]{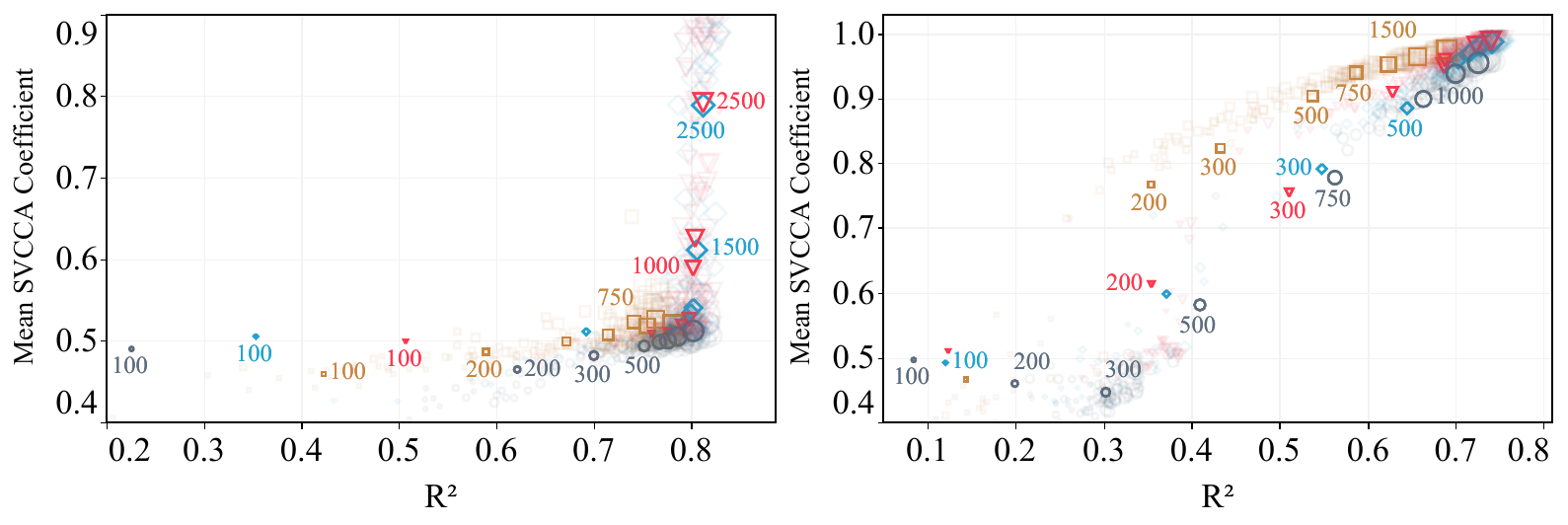}\vspace{-0.5em}
    \caption{Mean SVCCA coefficients and $R^2$ for the representation learners on the \clocks{} prediction task. The experiment was set up for sequences of length six ($T=5$), with outcomes of different dimensionality:  $q=1$ on the left and $q=2$ on the right. Solid marks represent the mean over 25 repetitions, while the faded marks denote the individual training runs. The annotations refer to the size of the training sets.}
    \label{fig:clocks_svcca_app_y1y2}
\end{figure}

\subsection{Square-Sign}

\begin{figure}[H]
\begin{subfigure}[t]{\textwidth}
\centering
    \includegraphics[width=\textwidth]{fig/legends/rf.pdf}
\end{subfigure}

\begin{subfigure}[t]{0.3\textwidth}
        \includegraphics[width=\textwidth]{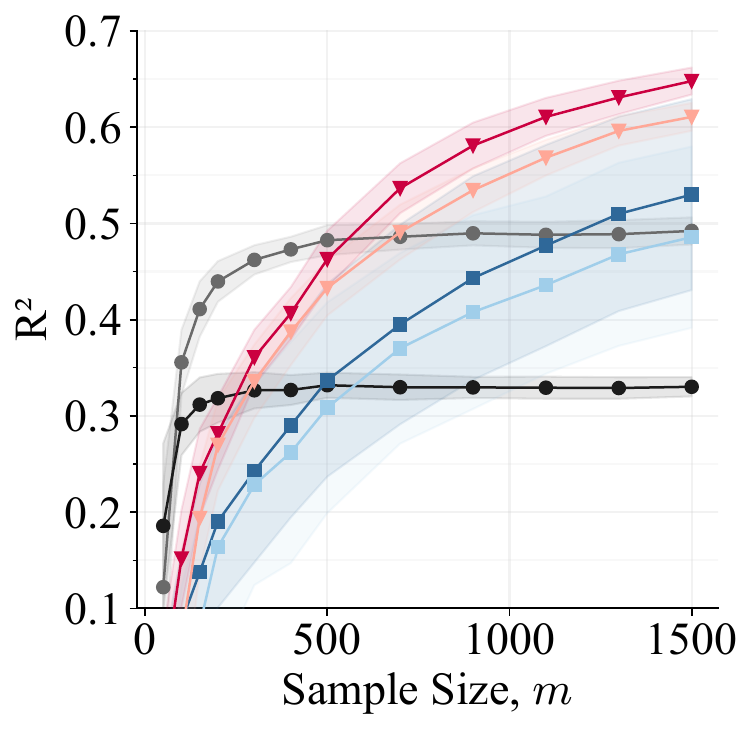}
            \caption{Sequences of length six, four privileged time points, $T=5$.}
\end{subfigure}
\hfill%
\begin{subfigure}[t]{0.3\textwidth}
\includegraphics[width=\textwidth]{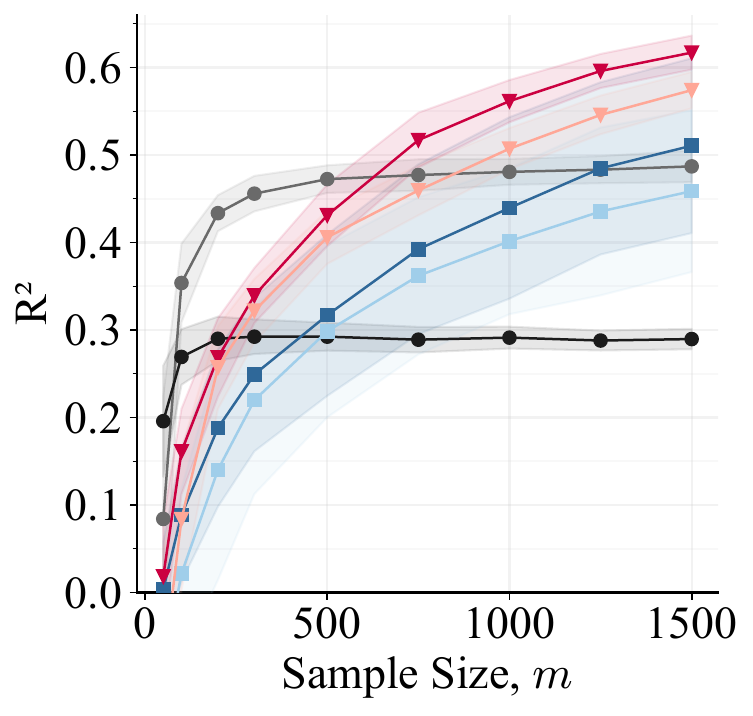}
    \caption{Sequences of length seven, five privileged time points, $T=6$.}
\end{subfigure}
\hfill
\begin{subfigure}[t]{0.3\textwidth}
\includegraphics[width=\textwidth]{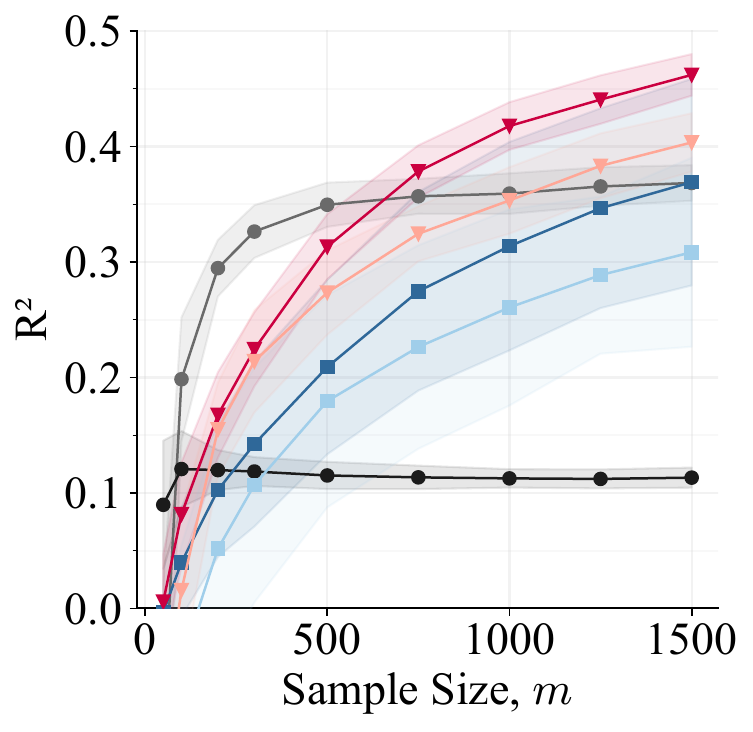}
    \caption{Sequences of length eight, six privileged time points, $T=7$.}
\end{subfigure}
    \caption{Predictive accuracy ($R^2$) of the different variants of generalized LuPTS applied to the prediction task offered by \squaresign{}. The DGP was configured with different sequence lengths and the experiments represent 60 repetitions.}
\end{figure}

\begin{figure}[H]
\begin{subfigure}[t!]{\textwidth}
\centering
    \includegraphics[width=\textwidth]{fig/legends/replearn_gd.pdf}
\end{subfigure}
\begin{subfigure}[t!]{0.3\textwidth}
        \includegraphics[width=\textwidth]{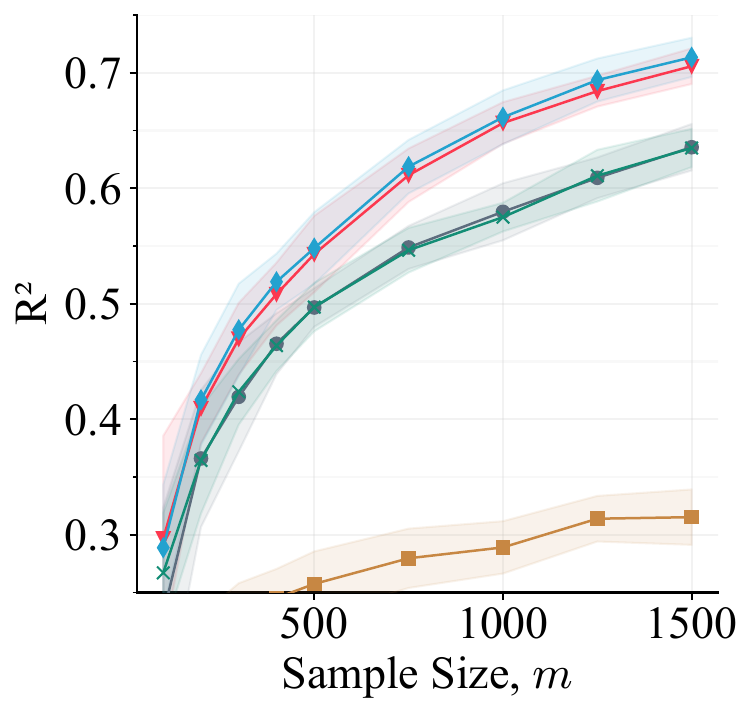}
            \caption{Sequences of length six, four privileged time points, $T=5$.}
\end{subfigure}
\hfill%
\begin{subfigure}[t!]{0.3\textwidth}
\includegraphics[width=\textwidth]{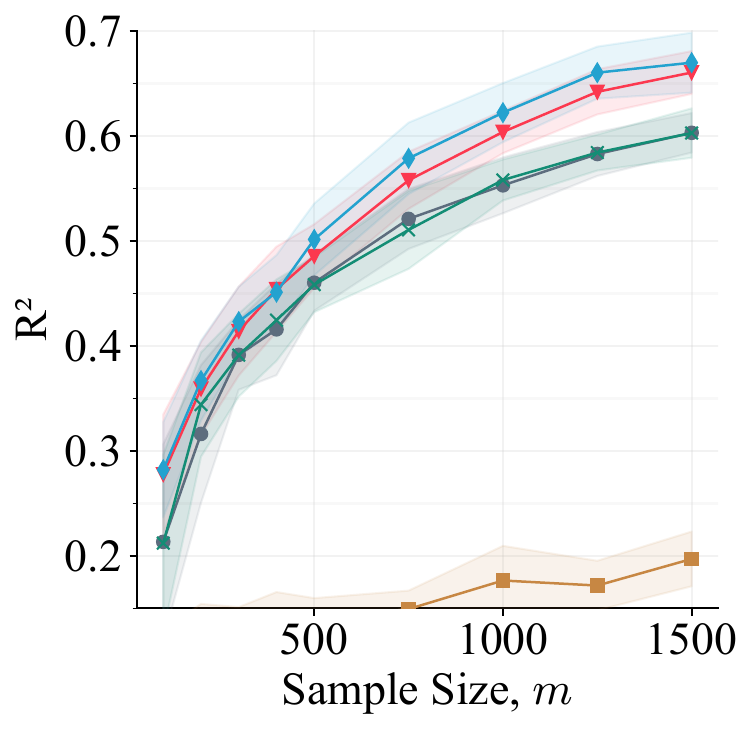}
    \caption{Sequences of length seven, five privileged time points, $T=6$.}
\end{subfigure}
\hfill
\begin{subfigure}[t!]{0.3\textwidth}
\includegraphics[width=\textwidth]{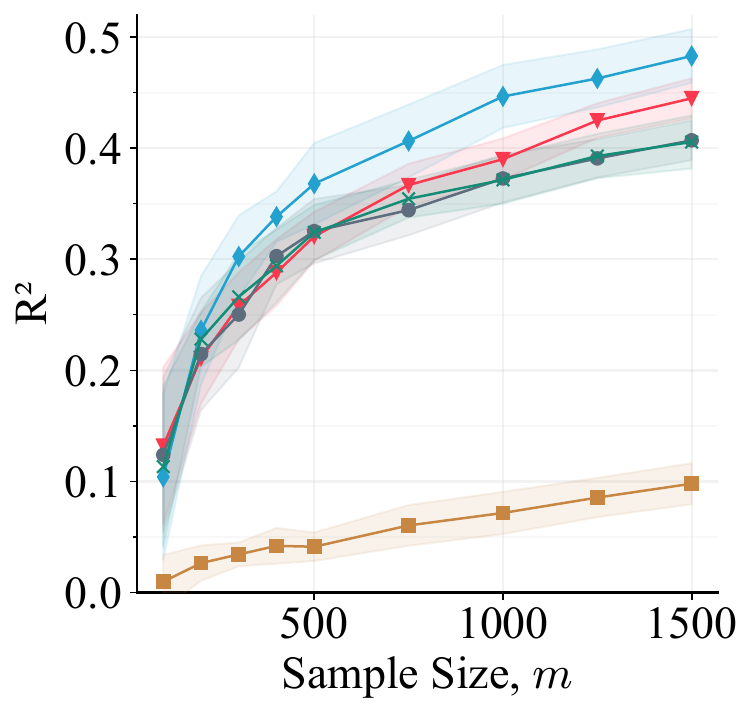}
    \caption{Sequences of length eight, six privileged time points, $T=7$.}
\end{subfigure}
    \caption{Predictive accuracy of the different representation learners introduced in Section~\ref{sec:learn_rep} when applied to the \squaresign{} data set for different sequence lengths. The experiments are based on 25 repetitions.}
\end{figure}
\begin{figure}[!h]
\begin{subfigure}[t!]{0.48\textwidth}
\centering
     \includegraphics[width=\textwidth]{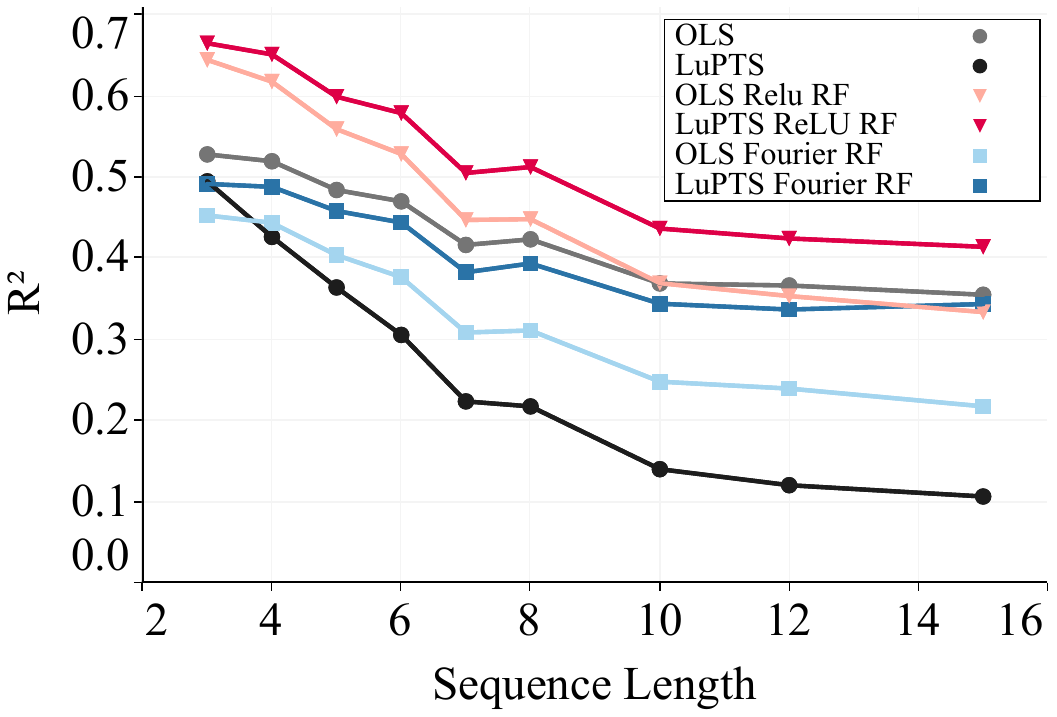}
     \fbox{\begin{minipage}[t][3cm][t]{1.0\linewidth}
    \caption{Mean $R^2$ over 500 repetitions for different sequence lengths on \squaresign{} ($d=10$, $q=3$) with a sample size of $m=1000$. Each run is performed on a different set of randomly sampled transition dynamics.}
      \label{fig:r2_sequence_length_square_sign}
  \end{minipage}}
\end{subfigure}
\hfill
\begin{subfigure}[t!]{0.48\textwidth}
\centering
    \includegraphics[width=\textwidth]{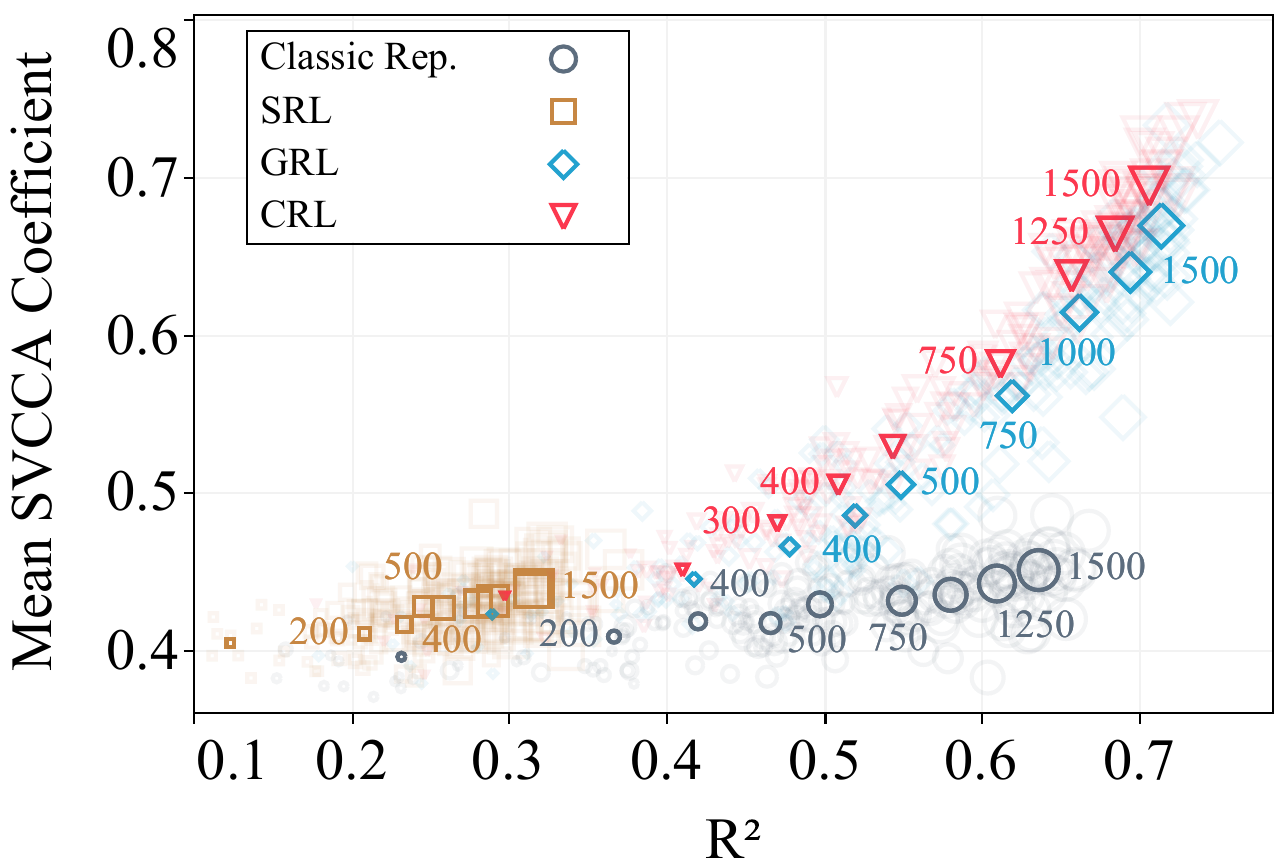}
     \fbox{\begin{minipage}[t][3cm][t]{1.0\linewidth}
    \caption{Mean SVCCA coefficients and $R^2$ for the representation learners on the \squaresign{} data set configured with ($T=5, d=10, q=3$). Solid marks represent the mean over 25 repetitions for a fixed  training set size while the faded marks denote individual runs. The annotations refer to the number of samples in the training data.}
  \end{minipage}}
\end{subfigure}
\caption{Generalized LuPTS applied on the \squaresign{} task for different sequence lengths (left panel), reporting the mean coefficient of determination $R^2$ over different systems. The right panel shows analysis of the predictive accuracy and latent recovery of the representation learners in the style of Figure~\ref{fig:clock_svcca_scatter_y1}.}
\label{fig:ss_svcca_scatter_app}
\end{figure}

\subsection{$\textbf{PM}_{2.5}$ air quality}

\begin{figure}[H]
\begin{subfigure}[t]{\textwidth}
\centering
    \includegraphics[width=\textwidth]{fig/legends/rf.pdf}
\end{subfigure}

\begin{subfigure}[t]{0.3\textwidth}
        \includegraphics[width=0.9\textwidth]{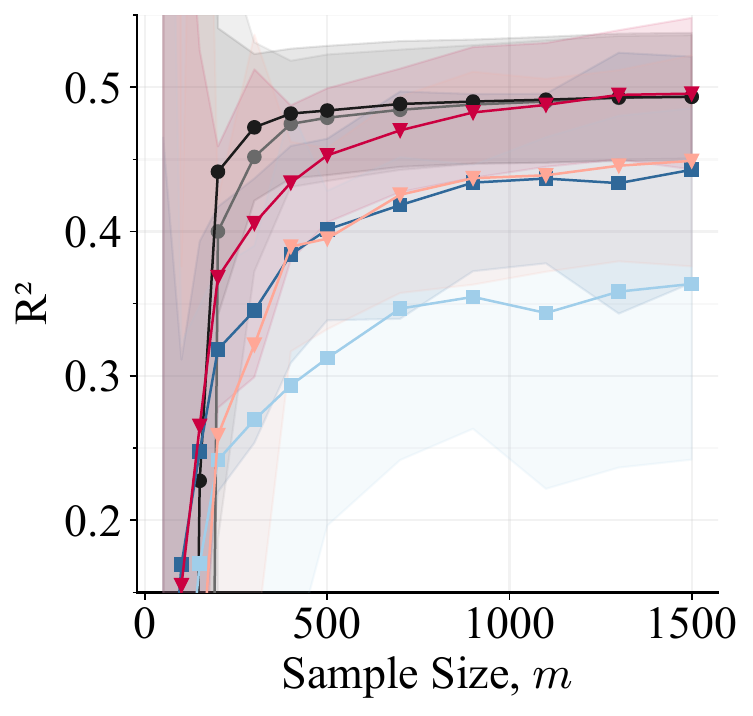}
            \caption{City of Beijing}
\end{subfigure}
\hfill%
\begin{subfigure}[t]{0.3\textwidth}
\includegraphics[width=0.9\textwidth]{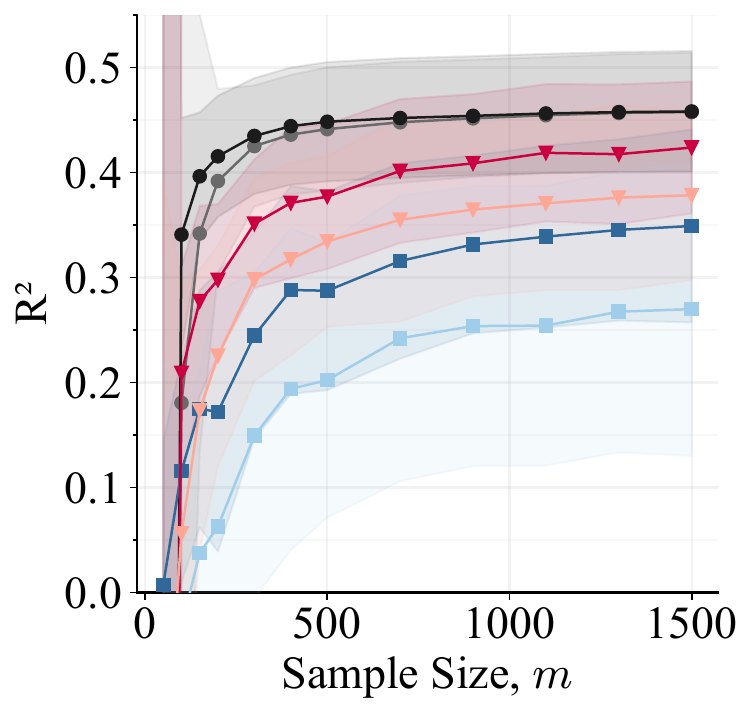}
    \caption{City of Shanghai}
\end{subfigure}
\hfill
\begin{subfigure}[t]{0.3\textwidth}
\includegraphics[width=0.9\textwidth]{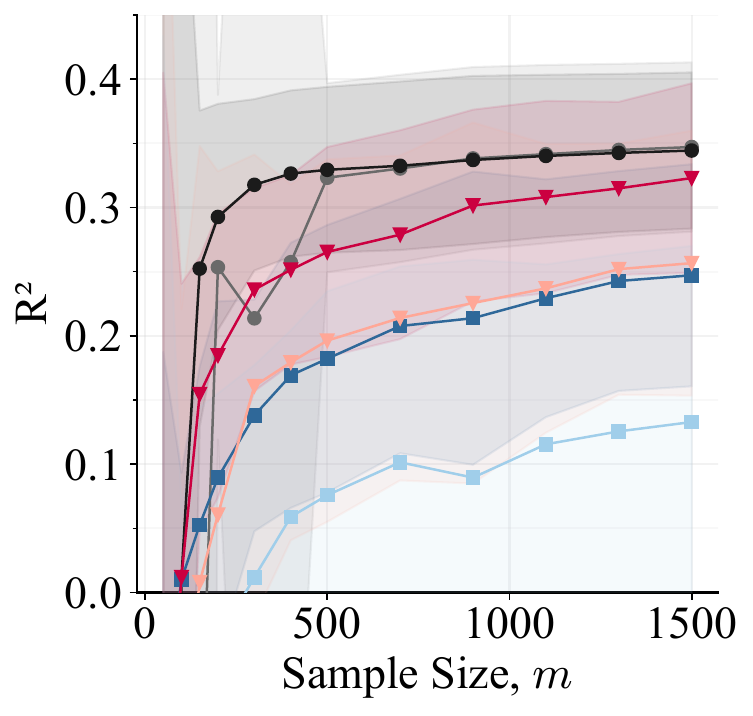}
    \caption{City of Shenyang}
    \vspace{.2em}
\end{subfigure}
\centering
\begin{subfigure}[t]{0.3\textwidth}
\includegraphics[width=0.9\textwidth]{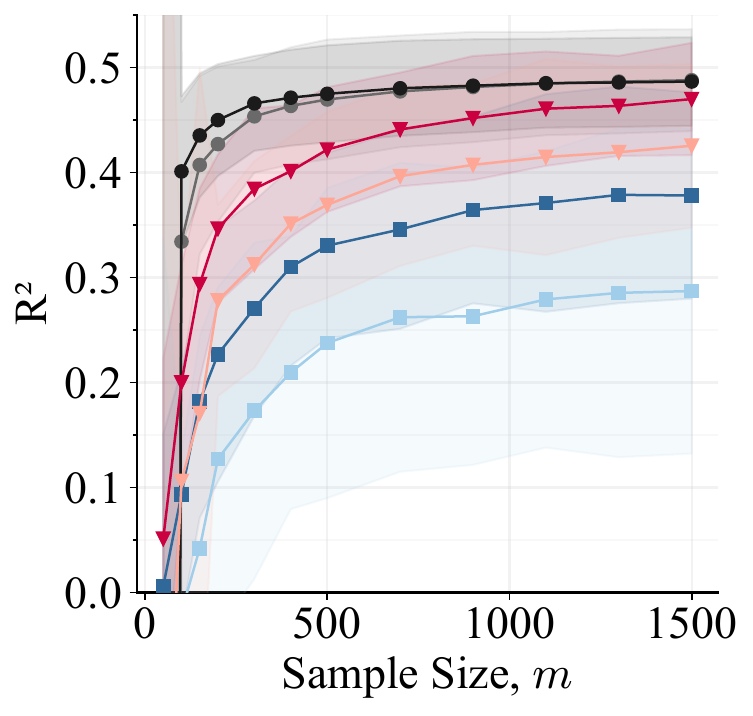}
    \caption{City of Guangzhou}
\end{subfigure}
\hspace{1em}
\begin{subfigure}[t]{0.3\textwidth}
\includegraphics[width=0.9\textwidth]{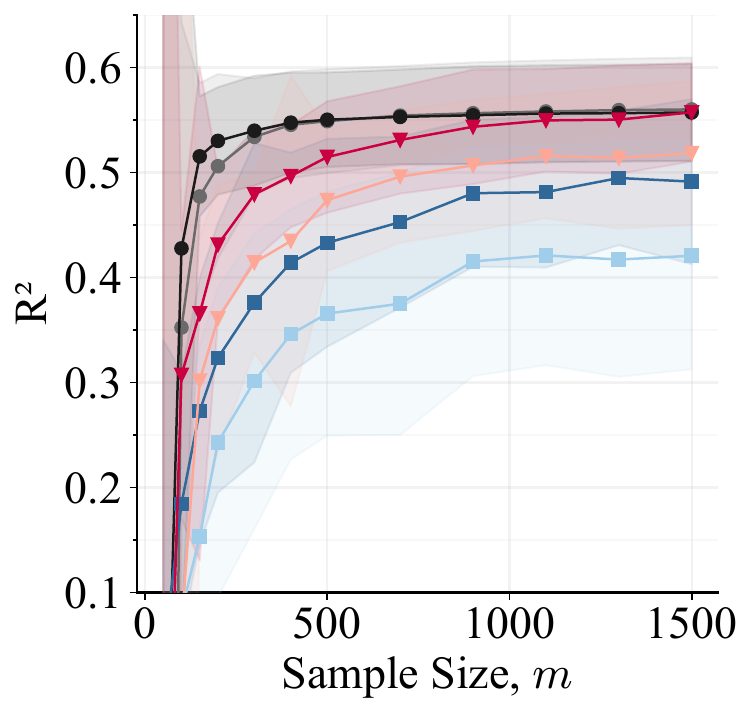}
    \caption{City of Chengdu}
\end{subfigure}
\caption{Predictive accuracy of the different generalized LuPTS variants on the prediction task posed by the \pmair{} data set. All experiments use time series of length five ($T=4$), where each time step is two hours long. The results were computed based on 60 repetitions.}
\end{figure}

\begin{figure}[H]
\begin{subfigure}[t]{\textwidth}
\centering
    \includegraphics[width=\textwidth]{fig/legends/replearn_gd.pdf}
\end{subfigure}

\begin{subfigure}[t]{0.3\textwidth}
        \includegraphics[width=0.9\textwidth]{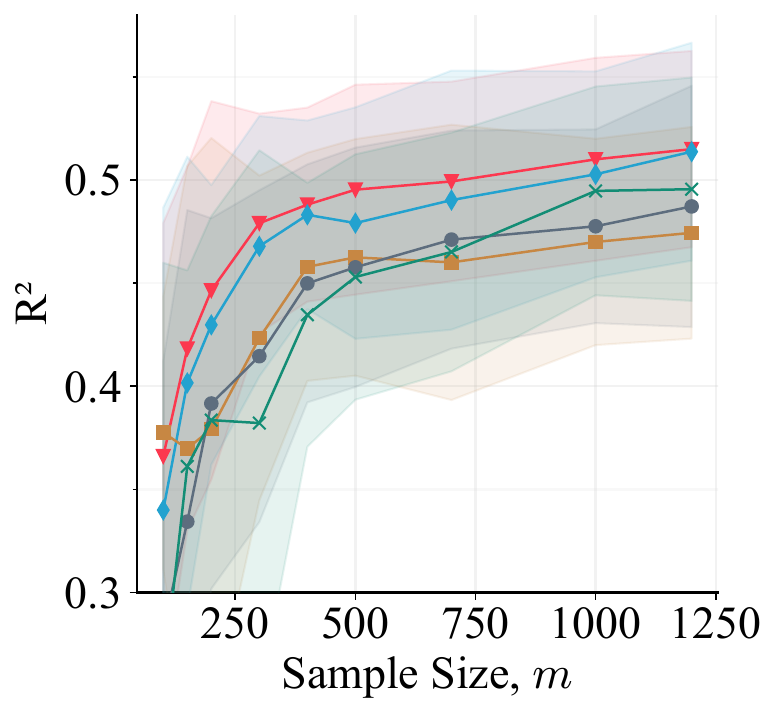}
            \caption{City of Beijing}
\end{subfigure}
\hfill%
\begin{subfigure}[t]{0.3\textwidth}
\includegraphics[width=0.9\textwidth]{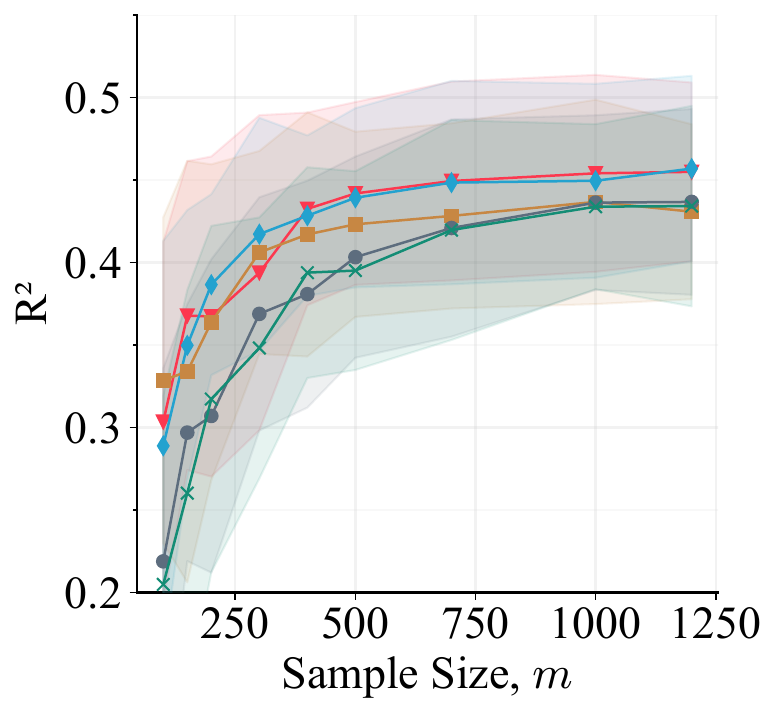}
    \caption{City of Shanghai}
\end{subfigure}
\hfill
\begin{subfigure}[t]{0.3\textwidth}
\includegraphics[width=0.9\textwidth]{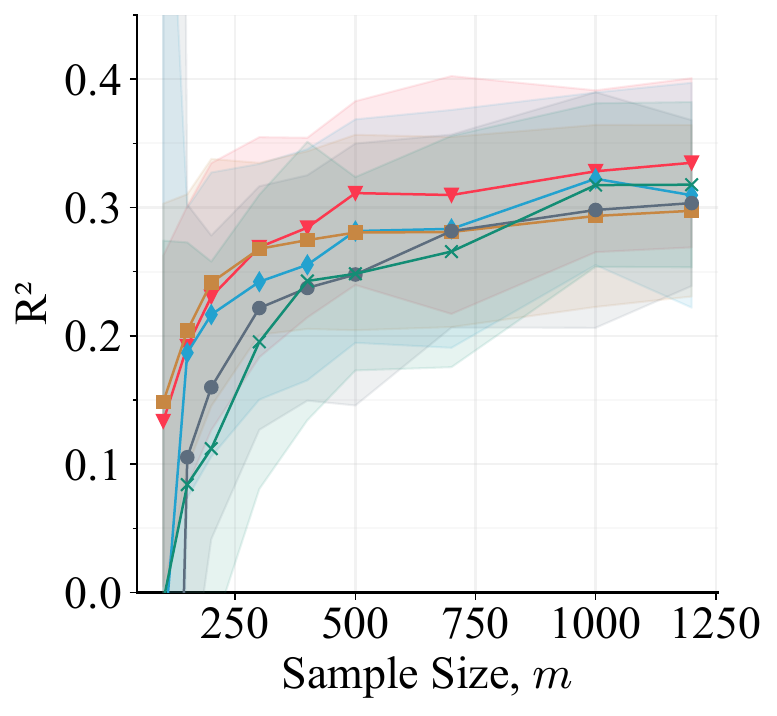}
    \caption{City of Shenyang}
    \vspace{.2em}
\end{subfigure}
\centering
\begin{subfigure}[t]{0.3\textwidth}
\includegraphics[width=0.9\textwidth]{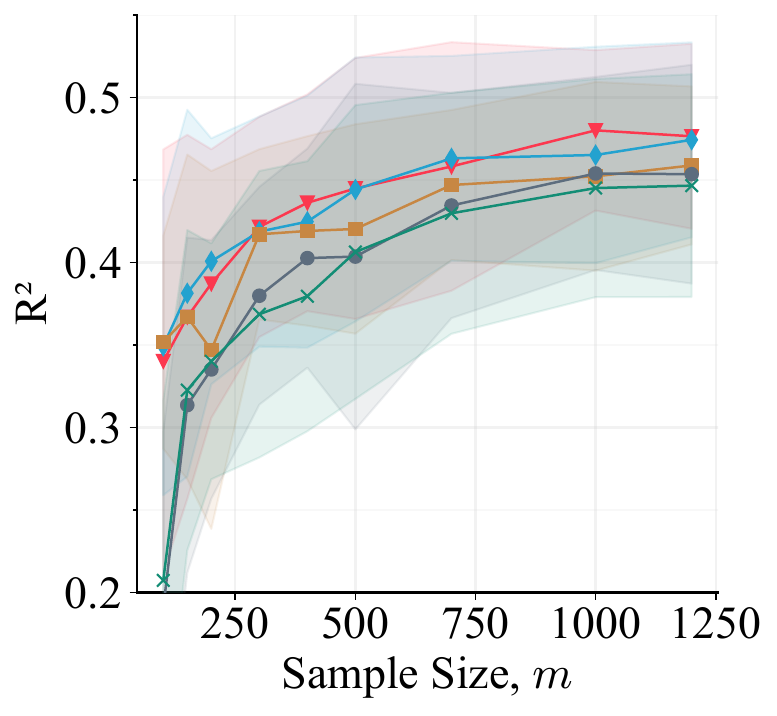}
    \caption{City of Guangzhou}
\end{subfigure}
\hspace{1em}
\begin{subfigure}[t]{0.3\textwidth}
\includegraphics[width=0.9\textwidth]{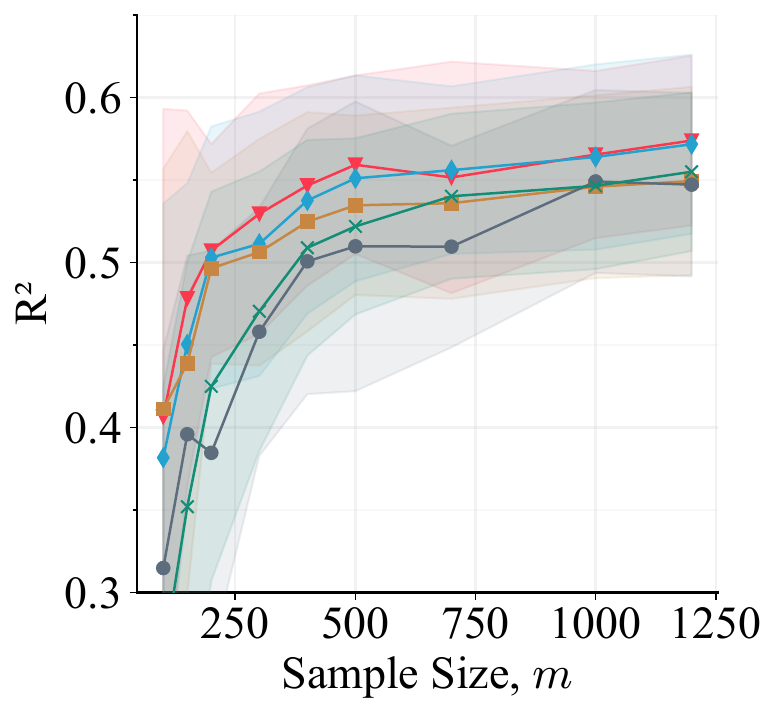}
    \caption{City of Chengdu}
\end{subfigure}
\caption{Evaluation of the sample efficiency of the represenation learning algorithms of Section~\ref{sec:learn_rep} on the \pmair{} air quality prediction task. All experiments use time series of length five ($T=4$), where each time step represents two hours. The results were computed based on 25 repetitions.}
\end{figure}

\end{document}